\documentclass[journal]{IEEEtran}
\usepackage{cite}
\usepackage{algorithm}
\usepackage{multirow}
\usepackage[table ]{xcolor}
\usepackage{array}
\usepackage{algpseudocode}
\usepackage{amsmath}
\usepackage{booktabs}
\usepackage{graphicx}
\usepackage{float}
\usepackage{threeparttable}
\usepackage{epstopdf}
\usepackage{xfrac}
\usepackage{makecell}
\usepackage{amssymb}
\usepackage{subfig}

\usepackage{mathrsfs} 

\usepackage{pifont}
\usepackage{amssymb}
\usepackage{longtable}

\usepackage{amsthm}
\newtheorem{definition}{Definition}

\newtheorem{theorem}{Theorem}
\newtheorem{example}{Example}
\newtheorem{lemma}{Lemma}

\begin{document}

\title{Runtime Analysis of Evolutionary Algorithms for Multi-party Multi-objective Optimization}

\author{Yuetong~Sun, Peilan~Xu,~\IEEEmembership{Member,~IEEE}, Wenjian~Luo,~\IEEEmembership{Senior Member,~IEEE}
    \thanks{This work has been accepted for publication in IEEE Transactions on Evolutionary Computation. The final published version will be available via IEEE Xplore.}
	\thanks{This work is partly supported by the Natural Science Foundation of Jiangsu Province (Grant No. BK20230419), Natural Science Foundation of the Jiangsu Higher Education Institutions of China (Grant No. 23KJB520018), National Natural	Science Foundation of China (Grant No. U23B2058).(\textit{Corresponding author: Peilan Xu.})}
	\thanks{Yuetong Sun and Peilan Xu are with School of Artificial Intelligence, Nanjing University of Information Science and Technology, Nanjing 210044, China (e-mail:202283460028@nuist.edu.cn; xpl@nuist.edu.cn).

            Wenjian Luo is with Guangdong Provincial Key Laboratory of Novel Security Intelligence Technologies, Institute of Cyberspace Security, School of Computer Science and Technology, Harbin Institute of Technology, Shenzhen 518055, Guangdong, China (e-mail:luowenjian@hit.edu.cn).}
	
}

\maketitle

\begin{abstract}
In scenarios where multiple decision-makers operate within a common decision space, each focusing on their own multi-objective optimization problem (e.g., bargaining games), the problem can be modeled as a multi-party multi-objective optimization problem (MPMOP). While numerous evolutionary algorithms have been proposed to solve MPMOPs, most results remain empirical. This paper presents the first theoretical analysis of the expected runtime of evolutionary algorithms on bi-party multi-objective optimization problems (BPMOPs). Our findings demonstrate that employing traditional multi-objective optimization algorithms to solve MPMOPs is both time-consuming and inefficient, as the resulting population contains many solutions that fail to achieve consensus among decision-makers. An alternative approach involves decision-makers individually solving their respective optimization problems and seeking consensus only in the final stage. While feasible for pseudo-Boolean optimization problems, this method may fail to guarantee approximate performance for one party in NP-hard problems. Finally, we propose evolutionary multi-party multi-objective optimizers (EMPMO) for pseudo-Boolean optimization and shortest path problems within a multi-party multi-objective context, maintain a common solution set among all parties. Theoretical and experimental results demonstrate that the proposed \( \text{EMPMO}_{\text{random}} \) outperforms previous algorithms in terms of the lower bound on the expected runtime for pseudo-Boolean optimization problems. Additionally, the consensus-based evolutionary multi-party multi-objective optimizer( \( \text{EMPMO}_{\text{cons}}^{\text{SP}} \) ) achieves better efficiency and precision in solving shortest path problems compared to existing algorithms.

\end{abstract}

\IEEEpeerreviewmaketitle

\section{Introduction}
\label{sec: introduction}

A significant proportion of multi-objective optimization problems (MOPs) involve two or more parties, and multi-objective games are one of typical scenarios \cite{roijers2013survey}, where multiple parties negotiate within the same decision space to reach a consensus or compromise. For example, UAV path planning requires consensus between government's concerns about urban safety and the enterprise's focus on economic benefits \cite{10463192}, the public facility construction involves the coordination among multiple functional departments \cite{zhang2012negotiation}, and the investment selection for commercial projects requires the evaluation of risks and benefits by various departments \cite{song2022multiobjective}. Liu \emph{et al.} \cite{liu2020evolutionary} defined such problems as multi-party multi-objective optimization problems (MPMOPs), which aim to find a solution set that satisfies or approximates the true Pareto front (PF) of all parties, corresponding to agreement or compromise through negotiation, respectively. 

Because each party in an MPMOP holds a separate MOP, multi-party multi-objective evolutionary algorithms (MPMOEAs) naturally evolve from multi-objective evolutionary algorithms (MOEAs). In other words, MPMOEAs are derived from MOEAs by further considering the relationships between multiple parties. Here, MPMOEAs refer to evolutionary algorithms involving M parties, where $M \geq 2$. In particular, when $M=2$, it is expressed as bi-party multi-objective evolutionary algorithms. The practice of addressing MPMOPs builds on the successful methodologies established in the evolutionary multi-objective optimization (EMO) community. Based on widely acclaimed MOEAs, MPMOEAs have been developed and can be classified into three categories \cite{chang2022multiparty}.
Dominance-based algorithms extend MOEAs such as the non-dominated sorting genetic algorithm II (NSGA-II) \cite{deb2002fast}, strength pareto evolutionary algorithm (SPEA2)~[\citen{song2022multiobjective},\citen{zitzler2001spea2}] , and multi-objective Immune Algorithm (MIA)~[\citen{10463192},\citen{lin2015hybrid}] by incorporating multi-party non-dominated sorting techniques~[\citen{liu2020evolutionary},\citen{she2021new}]. 
In this type of algorithm, each party performs non-dominated sorting of the population with respect to its own set of objectives, giving each individual a level per party. Then these per-party ranks are used jointly to define a multi-party dominance relationship. This allows selecting or ranking individuals in a way that reflects consensus or trade‐offs among all parties.
Decomposition-based methods leverage frameworks like the multi-objective evolutionary algorithm based on decomposition (MOEA/D) \cite{zhang2007moea}. Chang \emph{et al.} \cite{chang2022multiparty} proposed a party-by-party optimization strategy combined with MOEA/D, further refining the approach for solving bi-party multi-objective optimal power flow problems \cite{chang2023biparty}. Indicator-based algorithms, such as those proposed by Song \emph{et al.} \cite{song2024indicator}, integrate indicator-based multi-objective optimization techniques like SMS-EMOA \cite{beume2007sms} with multi-party non-dominated sorting. Overall, these studies illustrate the development of MPMOEAs, although they remain primarily empirical in nature. However, this also highlights the need for supplementary theoretical analysis.

Experimental results provide empirical support for the favorable performance of MPMOEAs, particularly for MPMOPs that involve common solutions (i.e., the intersection of the parties' Pareto-optimal solution sets). To the best of our knowledge, despite significant progress in evolutionary multi-objective optimization \cite{qian2022result, zheng2023mathematical, zheng2023runtime, opris2024runtime, ren2024first}, no theoretical analysis currently offers performance guarantees for evolutionary multi-party multi-objective optimization (EMPMO). 
This raises several key questions: Can traditional MOEAs effectively solve MPMOPs? Are their theoretical guarantees applicable in multi-party contexts? Furthermore, how can a simple and general framework for MPMOPs be designed to facilitate the theoretical analysis of evolutionary algorithms' approximation performance? It is worth noting that for MPMOPs with significant conflicts among parties, the set of common Pareto solutions is often empty, resulting in no common solution. Moreover, there is no clear definition of what constitutes an optimal solution in such cases without a common solution, whereas the case with a common solution is well-defined as the intersection of the individual Pareto fronts of all parties. Therefore, this paper focuses exclusively on scenarios where a common solution exists.

Early practical experience has demonstrated the limitations of MOEAs in addressing MPMOPs \cite{10463192}. Consider a bi-party bi-objective optimization problem where the Pareto sets of the two parties are denoted as \( \Phi_1 \) and \( \Phi_2 \), respectively. The common Pareto set, defined as the intersection of their Pareto-optimal sets, is given by \( \Phi_c = \Phi_1 \cap \Phi_2 \). For the corresponding four-objective optimization problem, that is, without considering multiple attributes, let its Pareto set be denoted as \( \Phi_3 \). Clearly, the cardinalities of these sets satisfy \( |\Phi_3| \geq \max\{|\Phi_1|, |\Phi_2|\} \geq |\Phi_c| \). This inequality highlights that if the multi-party attributes are ignored, the solutions obtained by the algorithm may include many Pareto-optimal solutions that fail to meet the definition of common Pareto optimality in the multi-party context.

Moreover, applying MOEAs to MPMOPs introduces additional computational overhead. For example, the runtime analysis of simple evolutionary multi-objective optimizer (SEMO) on pseudo-Boolean optimization functions such as COCZ (count ones count zeros) typically involves two stages \cite{laumanns2004running, qian2013analysis}. The first stage identifies a single Pareto-optimal solution starting from any arbitrary solution, while the second stage involves discovering the complete Pareto set \( \Phi \) starting from one Pareto-optimal solution. The runtime of the second stage heavily depends on the cardinality of the Pareto set \( |\Phi| \). Moreover, in the theoretical analysis of MOEAs, each stage considers all possible non-dominated solutions to compute transition probabilities \cite{horoba2010exploring}. When multi-party attributes are ignored, the number of non-dominated solutions in the population grows exponentially, potentially worsening the expected runtime of the algorithm significantly.

Based on the definition of multi-party common Pareto-optimal solutions, a straightforward MPMOEAs can be derived from MOEAs. Taking a bi-party multi-objective optimization problem as an example, we first use MOEAs to obtain the individual Pareto solution sets \( \Phi_1' \) and \( \Phi_2' \) for the two parties. The common Pareto solution set is then computed as their intersection \( \Phi_c' = \Phi_1' \cap \Phi_2' \). This approach leverages existing multi-objective optimization theories without introducing additional conceptual frameworks. It is evident that if the complete Pareto solution sets are obtained, i.e., \( \Phi_1' = \Phi_1 \) and \( \Phi_2' = \Phi_2 \), then \( \Phi_c' = \Phi_c \). However, for NP-hard problems, exact solutions cannot be found in polynomial time \cite{koza2003s, benkhelifa2009design, yu2012approximation} unless P=NP. Consequently, this method may result in an empty intersection, making it impossible to identify a solution, or require one party to further compromise on the quality of the solution.

Overall, our contributions are that we propose a series of evolutionary multi-party multi-objective optimizers (EMPMO) for solving MPMOPs (with a focus on the bi-party case for simplicity),  in order to fill the theoretical analysis gap between MOEAs and MPMOEAs. These include both an artificial multi-party multi-objective problem and an NP-hard multi-party multi-objective shortest path problem (MPMOSP).
We first prove that the transitional algorithm \(\text{EMPMO}_{\text{simple}}\) achieves the common Pareto solutions with a runtime bound of \( O(3n^2\log n) \), yet exposes \( (1 + \varepsilon) \)-approximation limitations for NP-hard problems, which is to solve each party’s PS separately and then return the intersection. To overcome this, for the artificial problem, we introduce two EMPMOs: \( \text{EMPMO}_{\text{random}} \) employs an alternating evolution strategy with the runtime bounded by \( O\left(\left(\frac{1}{\varphi} + \frac{1}{1-\varphi} \right) \frac{n^2}{2} \log n \right) \), where $\varphi$ is the probability of choosing the first party, and \( \text{EMPMO}_{\text{payoff}} \) achieves a runtime bound of \( O\left( n \log n \right) \) for this problem.

For the MPMOSP problem, we demonstrate that it satisfies the multi-party optimal substructure property and design \( \text{EMPMO}_{\text{cons}}^{\text{SP}} \), which achieves a \( (1 + \varepsilon) \)-approximation with an expected running time bounded by
\[
O\left(n^4 \cdot \delta^{k-1} \cdot \log\left(n\delta^{k-1}\right) \right).
\]
where $1+\varepsilon$ represent the approximation ratio, $k$ is the maximum number of objectives among all parties, $w^{\max}$ is the maximum weights across all parties, and $\delta=\frac{n \log \left(n w^{\max}\right)}{\log(1+\varepsilon)}$.

Experiments confirm that \( \text{EMPMO}_{\text{random}} \)has near-optimal performance at balanced selection ($\varphi=0.5$), approaching the performance of\( \text{EMPMO}_{\text{payoff}} \). The experiment on a real-world bi-party UAV path planning problem also validate the feasibility of \( \text{EMPMO}_{\text{cons}}^{\text{SP}} \)and show the limitations of \( \text{EMPMO}_{\text{simple}}^{\text{SP}} \) .

The rest of this paper is organized as follows. Section \ref{sec:preliminaries} introduces the preliminaries. Section \ref{sec: theoretical analysis} provides the theoretical analysis for the artificially constructed pseudo-Boolean optimization problem. Section \ref{sec: theoretical analysis2} also presents the theoretical analysis for the bi-party multi-objective shortest path problem. Section \ref{sec:experiments} presents experiments that complement the theoretical analysis. Finally, Section \ref{sec: Conclusion} concludes the paper and discusses future work. 

\section{Preliminaries}
\label{sec:preliminaries}
In this section, we give brief introductions to multi-party multi-objective optimization problems (MPMOPs) and simple evolutionary multi-objective optimizer. In this article, for the pseudo-Boolean optimization problem, we used the number of mutations needed to cover the Pareto front to measure the running time, and for the bi-party multi-objective shortest path problem (BPMOSP), we counted the number of generations.

\subsection{Multi-party Multi-objective Optimization Problems}
In real-world scenarios, many optimization problems involve multi-party games, which each party has its own set of objectives that may conflict with those of others. These problems are known as MPMOPs, and represent an extension of classical MOPs, where a single party simultaneously optimizes multiple objectives. In contrast, MPMOPs consider a setting with multiple parties, each having distinct objectives and priorities. Without loss of generality, a minimization MPMOP can be defined as follows:

\begin{definition}[MPMOP \cite{liu2020evolutionary}] \label{def:MPMOP}
	Let $M$ denote the number of parties. For each party $m = 1, 2, \ldots, M$, let $k_m$ denote the number of objectives they are concerned with. The MPMOP is defined as:
	\begin{equation}\label{eqt:MPMOPs}
		\min_{\mathbf{x} \in X } \mathcal{F}(\mathbf{x}) = \left(F_1(\mathbf{x}), F_2(\mathbf{x}), \ldots, F_M(\mathbf{x})\right),
	\end{equation}
	where the objective functions $F_m(\mathbf{x})$ for each party $m$ are given by:
	\begin{equation}\label{eqt:MOPs}
		F_m(\mathbf{x}) = \left(f_{m1}(\mathbf{x}), f_{m2}(\mathbf{x}), \ldots, f_{mk_m}(\mathbf{x})\right),  m = 1, 2, \ldots, M,
	\end{equation}
	where $X$ is the feasible solution space, $\mathbf{x} \in X $ is the decision vector, bounded by lower and upper limits $x_{\min}$ and $x_{\max}$, and $\mathcal{F}(\mathbf{x}) = \left(F_1(\mathbf{x}), F_2(\mathbf{x}), \ldots, F_M(\mathbf{x})\right)$ is the vector of objective functions from all parties.
\end{definition}

In multi-objective optimization problems (MOPs) of each party, dominance relations (Definition \ref{def:Domination}) for minimization problems are introduced to assess solution quality, forming the foundation for defining Pareto optimality (Definition \ref{def:Pareto}).

\begin{definition}[Domination] \label{def:Domination}
	Let \( F_m = (f_1, f_2, \ldots, f_k) : X \to \mathbb{R}^k \) be the objective vector of any party $m$, where \( X \) is the feasible solution space. For two solutions \( \mathbf{x}, \mathbf{x}' \in X \):
	\begin{enumerate}
		\item \( \mathbf{x} \) weakly dominates \( \mathbf{x}' \) if \( f_j(\mathbf{x}) \leq f_j(\mathbf{x}') \) for all \( j \in \{1, \ldots, k\} \), denoted as \( \mathbf{x} \succeq \mathbf{x}' \).
		\item \( \mathbf{x} \) dominates \( \mathbf{x}' \) if \( \mathbf{x} \) weakly dominates \( \mathbf{x}' \) and \( f_j(\mathbf{x}) < f_j(\mathbf{x}') \) for at least one \( j \), denoted as \( \mathbf{x} \succ \mathbf{x}' \).
	\end{enumerate}
\end{definition}

\begin{definition}[Pareto Optimality] \label{def:Pareto}
	Let \( F_m = (f_1, f_2, \ldots, f_k) : X \to \mathbb{R}^k \) be the objective vector of any party $m$, where \( X \) is the feasible solution space. A solution \( \mathbf{x} \) is Pareto optimal if no other solution \( \mathbf{x}' \in X \) satisfies \( \mathbf{x}' \succ \mathbf{x} \).
	
	A set \( \Phi_m =\{\mathbf{x}\in X\mid\nexists\mathbf{x}^{\prime}\in X\mathrm{~with~}\mathbf{x}^{\prime}\succ\mathbf{x}\} \subset X \) is called a Pareto set(PS), which contains only all the Pareto optimal solutions on this party. $F_m^*=F_m(\Phi_m)$ is the collection of objective values corresponding to \( \Phi_m \) and is called Pareto front(PF).
\end{definition}

Building on these concepts, the notion of common Pareto optimality is introduced for multi-party multi-objective optimization problems (MPMOPs) (Definition \ref{def:CommonPareto}).

\begin{definition}[Common Pareto Optimality \cite{liu2020evolutionary}] \label{def:CommonPareto}
	Let \( \Phi_m \) represents the Pareto set of party \( m \). For all $m \in \{1,2,\cdots,M\}$, the set
	\[
	\Phi = \bigcap_{m=1}^{M} \Phi_m,
	\]
	is referred to as the common Pareto set. Then a solution \( \mathbf{x}^* \in \Phi \subset X \) is called common Pareto optimal.
\end{definition}

A common Pareto optimal solution represents an equilibrium where the objectives of all parties are simultaneously optimized, rendering it inherently acceptable to all without the need for further compromise. Such a solution is universally optimal across all parties, distinguishing it from non-common solutions, which may satisfy only a subset of the parties' objectives and often require additional negotiations or adjustments to achieve broader consensus. It is worth noting that this paper focuses exclusively on MPMOPs in which common Pareto optimal solutions exist.

\subsection{Simple Evolutionary Multi-objective Optimizer}
Simple evolutionary multi-objective optimizer (SEMO) \cite{laumanns2004running} represents the simplest instance of a population-based evolutionary algorithm for multi-objective optimization and is commonly employed in theoretical analyses on multi-objective optimization \cite{friedrich2007approximating, giel2006effect, qian2011analysis, allmendinger2015multiobjective}. SEMO maintains a population of unfixed size to store all solutions that are not dominated by any other solution encountered thus far. The algorithm begins by randomly selecting an initial solution from the decision space to initialize the population. In each iteration, a solution is randomly chosen from the current population to undergo a one-bit mutation. The resulting mutated solution is then compared against all existing solutions in the population. Dominated solutions are subsequently removed to refine and maintain the population's quality.

\begin{algorithm}[!h]
	\caption{Simple Evolutionary Multi-objective Optimizer (SEMO)}
	\label{alg:SEMO}
	\begin{algorithmic}[1]
		\State Choose an individual $\mathbf{x}$ uniformly from $X$;
		\State $P\leftarrow \{ \mathbf{x}\}$;
		\While {termination condition is not met}
		\State Choose a solution $\mathbf{x}$ uniformly at random from $P$;
		\State $\mathbf{x}'$ is generated by one-bit mutation to $\mathbf{x}$;
		\If{$\nexists \mathbf{z}\in P\: \mathrm{ satisfying}\:(\mathbf{z}\succ \mathbf{x}'\lor f(\mathbf{z})=f(\mathbf{x}'))$}
		\State $P\leftarrow(P \setminus \{\mathbf{z}\in P \mid \mathbf{x}'\succ \mathbf{z}\})\cup\{\mathbf{x}'\}$;
		\EndIf
		\EndWhile
	\end{algorithmic}
\end{algorithm}

\section{Runtime Analysis of Evolutionary Algorithm on Artificial Multi-party Multi-objective Problem}
\label{sec: theoretical analysis}

We initiate a theoretical investigation into evolutionary multi-party multi-objective optimizer by employing a pseudo-Boolean function. To the best of our knowledge, this is the first study addressing multi-party multi-objective optimization problems. As the foundation for our analysis, we construct a bi-party bi-objective pseudo-Boolean problem, termed BPAOAZ (bi-party all ones all zeros), tailored to this context. Our work begins by deriving the simple evolutionary multi-party multi-objective optimizer from the well-known SEMO and analyzes its runtime performance. Building on this baseline, we introduce two variants, i.e., random and payoff-based evolutionary multi-party multi-objective optimizer, inspired by the concepts of ``bounded rationality" and ``full rationality" from game theory, and provide their respective runtime complexity analysis in solving the BPAOAZ problem. Finally, we demonstrate a runtime performance comparison between these evolutionary bi-party multi-objective optimizers and SEMO, where BPAOAZ is treated as a standard multi-objective optimization problem by the latter. In these algorithms, the population is explicitly or implicitly divided into several subpopulations, each tasked with exploring the Pareto set of a specific decision-making party. Ultimately, the algorithms identify the common Pareto set through an evolutionary process. Therefore, we refer to these algorithms as evolutionary algorithms \cite{potter1994cooperative, ficici2000game}.

\subsection{Artificial Bi-party Bi-objective Optimization Problem}

Bi-objective pseudo-Boolean optimization problems, such as LOTZ, COCZ, and OneMinMax, are widely used to analyze the performance of evolutionary multi-objective optimizers \cite{laumanns2004running, giel2010effect, zheng2023runtime}. Building upon these foundational bi-objective test cases, we propose an artificial problem referred to as BPAOAZ (bi-party all ones all zeros), designed for two parties, each with two objectives, and ensure the existence of at least one solution in the Pareto optimal set for both parties.

\begin{definition}[BPAOAZ]\label{def: BPAOAZ}
	The pseudo-Boolean function BPAOAZ : $\{0,1\}^n \to \mathbb{N}^2 \times \mathbb{N}^2$ is defined as
	\[	
	\mathrm{BPAOAZ}(\mathbf{x})=\left(\mathrm{AORZ}(\mathbf{x}), \mathrm{AOFZ}(\mathbf{x})\right),
	\]
	where \(\mathbf{x} = (x_1, x_2, \ldots, x_n) \in \{0,1\}^n\), and \(n = 2a\) for some \(a \in \mathbb{N}\). The components \(\mathrm{AORZ}(\mathbf{x})\) and \(\mathrm{AOFZ}(\mathbf{x})\) are given as follows:
	\[
	\mathrm{AORZ}(\mathbf{x}) = \big(f_{11}(\mathbf{x}), f_{12}(\mathbf{x})\big),
	\]
	where
	\[
	f_{11}(\mathbf{x}) = \sum_{i=n/2+1}^n x_i, \quad
	f_{12}(\mathbf{x}) = \sum_{i=1}^{n/2} x_i + \sum_{i=n/2+1}^n (1 - x_i);
	\]
	
	\[
	\mathrm{AOFZ}(\mathbf{x}) = \big(f_{21}(\mathbf{x}), f_{22}(\mathbf{x})\big),
	\]
	where
	\[
	f_{21}(\mathbf{x}) = \sum_{i=1}^{n/2} (1 - x_i) + \sum_{i=n/2+1}^n x_i, \quad
	f_{22}(\mathbf{x}) = \sum_{i=1}^{n/2} x_i.
	\]
\end{definition}

The problem BPAOAZ consists of two bi-objective pseudo-Boolean optimization problems, AORZ (all ones rear zeros) and AOFZ (all ones front zeros), whose decision space cardinality is $2^n$. In AORZ, the first objective is to maximize the number of ones in the second half of a solution, while the second objective is to maximize the sum of the ones in the first half and the zeros in the second half. These objectives are conflicting in the second half of the solution. In AOFZ, the first objective is to maximize the sum of the zeros in the first half and the ones in the second half, and the second objective is to maximize the number of ones in the first half of a solution. These two objectives are conflicting in the first half of the solution. Thus, BPAOAZ involves conflicts between the objectives of the two parties.

The objective space of AORZ can be partitioned into $\frac{n}{2}+1$ subsets $F_{1,i}$ (Fig. \ref{fig:AORZ}), where $i \in \{0,1,...,\frac{n}{2}\}$ represents the number of ones in the first half of the solution. Each subset $F_{1, i}$ contains $\frac{n}{2}+1$ objective vectors of the form $(j, i+\frac{n}{2}-j)$, where $j \in \{0,1,...,\frac{n}{2}\}$ is the number of ones in the second half of the solution. Notably, $F_{1,\frac{n}{2}}=\{(\frac{n}{2},\frac{n}{2}),(\frac{n}{2}-1,\frac{n}{2}+1),...,(0,n)\}$ represents the PF $F_1^{*}$, with cardinality $|F_1^{*}|=|F_{1,\frac{n}{2}}|=\frac{n}{2}+1$. The subdomains $X_{1,i}$ are defined as the sets of all decision vectors mapped to elements of $F_{1,i}$. The Pareto set $X_1^{*}=X_{1, \frac{n}{2}}=\{1^{\frac{n}{2}}a, a\in \{0,1\}^{\frac{n}{2}}\}$ has cardinality $ |X_{1,\frac{n}{2}}| = 2^{\frac{n}{2}}$. The entire decision space $X_1$ contains $2^n$ elements.

\begin{figure}
	\centering
	\small
	\subfloat[AORZ.]{\includegraphics[width=0.45\columnwidth]{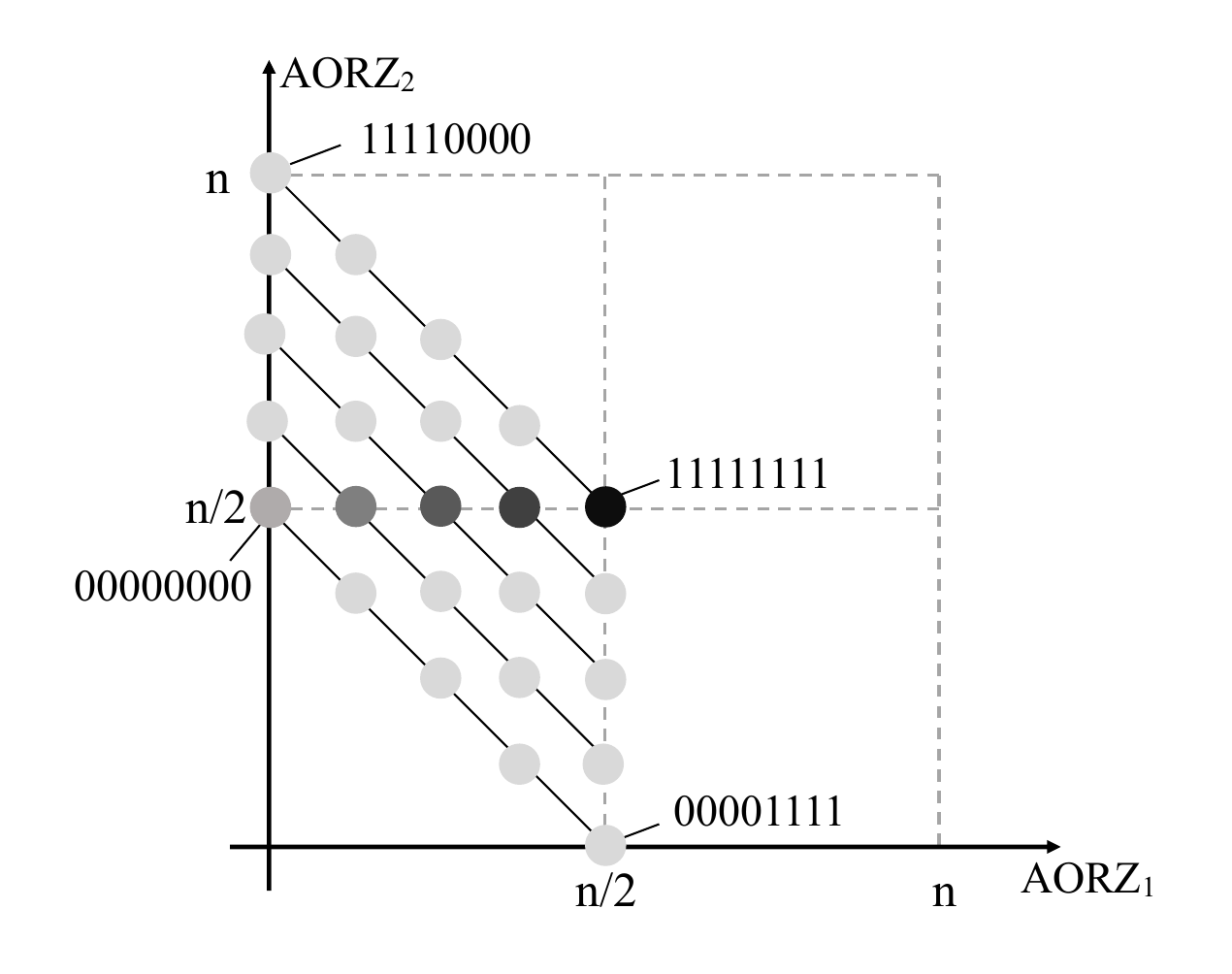}
		\label{fig:AORZ}}
	\subfloat[AOFZ.]{\includegraphics[width=0.45\columnwidth]{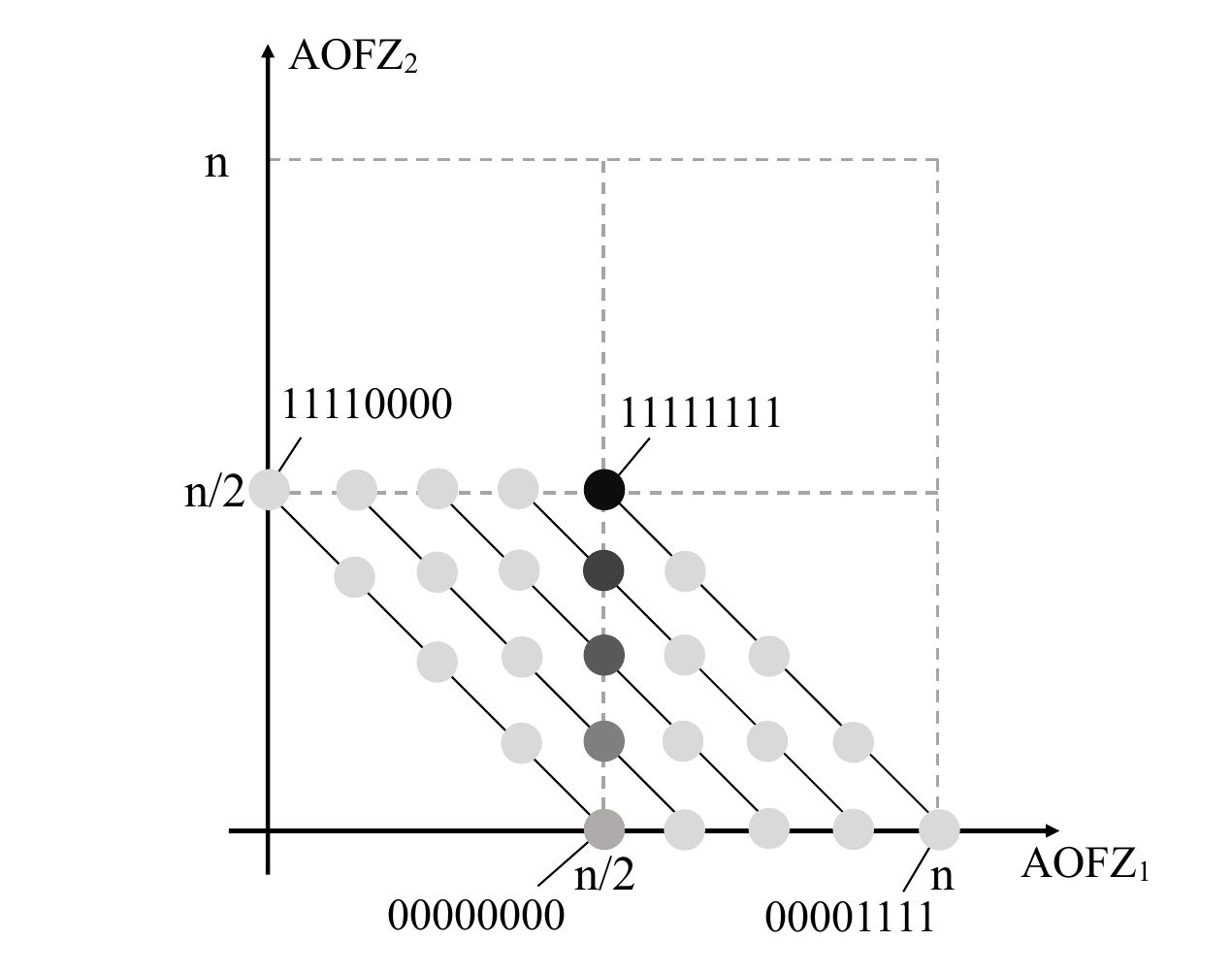}
		\label{fig:AOFZ}}
	\caption{The objective space of the AORZ and AOFZ problem with $n=8$.}
	\label{fig:BPAOAZ}
\end{figure}

The objective space of AOFZ can be partitioned into $\frac{n}{2}+1$ subsets $F_{2, j}$ (Fig. \ref{fig:AOFZ}), where $j \in \{0,1,...,\frac{n}{2}\}$ represents the number of ones in the second half of the solution. Each subset $F_{2, j}$ contains $\frac{n}{2}+1$ objective vectors of the form $(\frac{n}{2}-i+j, i)$, where $i \in \{0,1,...,\frac{n}{2}\}$ is the number of ones in the first half of the solution. Notably, $F_{2, \frac{n}{2}}=\{(\frac{n}{2},\frac{n}{2}),(\frac{n}{2}+1,\frac{n}{2}-1),...,(n,0)\}$ represents the PF $F_2^{*}$, with cardinality $|F_2^{*}|=|F_{\frac{n}{2}}|=\frac{n}{2}+1$. The Pareto set $X_2^{*}= X_{2,\frac{n}{2}}  =\{a1^{\frac{n}{2}},a\in \{0,1\}^{\frac{n}{2}}\}$ has cardinality $ |X_{2,\frac{n}{2}}|= 2^{\frac{n}{2}}$. The entire decision space $X_2$ contains $2^n$ elements.



Consider the BPAOAZ problem. Let us take Figure \ref{fig:AORZ} as an example for analysis, with Figure \ref{fig:AOFZ} exhibiting similar properties. In Figure \ref{fig:AORZ}, the points that are darker than their surroundings represent the common solutions for each layer, denoted as \( X_{1,i} \cap X_{2,i} \), where the solution \(\mathbf{x} = (x_1, x_2, \ldots, x_n) \in \{0,1\}^n\) satisfies $\sum_{i=1}^{n/2} x_i=\sum_{i=n/2+1}^n x_i$. Due to the dominance relationship between \( F_{1,i} \) and \( F_{2,i} \) within each layer, the common solutions at a higher layer dominate those at a lower layer for both parties. Specifically, among the five darker points in Figure \ref{fig:AORZ}, the darker points dominate the lighter points for both parties. The darkest point in the figure represents the common Pareto optimal solution. Consequently, there exists only one common Pareto optimal solution for BPAOAZ, which is \( 1^n \). The corresponding common PF is unique and is given by \( F^* = \left\{ \left( \frac{n}{2}, \frac{n}{2} \right) \right\} \).

\subsection{Runtime Analysis of Evolutionary Multi-party Multi-objective Optimizer}\label{22}

Based on Definition \ref{def:CommonPareto} regarding the common Pareto set in MPMOPs, a straightforward approach can be derived from SEMO for solving such problems, referred to as the simple evolutionary multi-party multi-objective optimizer ($\text{EMPMO}_{\text{simple}}$), which is outlined in Algorithm \ref{alg:MPCOEAs}. The $\text{EMPMO}_{\text{simple}}$ operates by independently identifying the Pareto sets for each party and subsequently determining their intersection. Its feasibility depends on two critical assumptions: 1) the presence of at least one solution common to the Pareto optimal sets of both parties and 2) the ability to obtain the complete Pareto set for each evolutionary multi-objective optimizer.

\begin{algorithm}[!h]
	\caption{ Simple Evolutionary Multi-party Multi-objective Optimizer ($\text{EMPMO}_{\text{simple}}$) }
	\label{alg:MPCOEAs}
	\begin{algorithmic}[1]
		\State Randomly select an individual $\mathbf{x} \in X$;
		\State $\Phi \gets \{\mathbf{x}\}$;
		\For {$m \gets 1$ \textbf{to} M}
		\State Set initial population $P_m \leftarrow \{ \mathbf{x}\}$ for each party;
		\EndFor
		\While {termination condition is not met}
		\For {$m \gets 1$ \textbf{to} $M$}
		\State Randomly select an individual $\mathbf{x} \in P_m$;
		\State Apply one-bit mutation on $\mathbf{x}$ to generate $\mathbf{x}'$;
		\If{$\nexists \, \mathbf{z} \in P_m$ \textbf{such that} $\mathbf{z} \succ_m \mathbf{x}'$ \textbf{or} $F_m(\mathbf{z}) = F_m(\mathbf{x}')$}
		\State $P_m \gets (P_m \setminus \{\mathbf{z} \in P_m \mid \mathbf{x}' \succ_m \mathbf{z}\}) \cup \{\mathbf{x}'\}$
		\If{ $\nexists \, \mathbf{z} \in \Phi$ \textbf{such that} $\forall m'\ \mathbf{z} \succ_{m'} \mathbf{x}'$ \textbf{or} $F_{m'}(\mathbf{z}) = F_{m'}(\mathbf{x}')$  }
		\State $\Phi \gets (\Phi \setminus \{ \mathbf{z} \in \Phi \mid \exists m''\ \mathbf{x}' \succ_{m''} \mathbf{z}\}) \cup \{\mathbf{x}'\}$
		\EndIf
		\EndIf
		\EndFor
		\EndWhile
	\end{algorithmic}
\end{algorithm}

The \( \text{EMPMO}_{\text{simple}} \) algorithm begins by randomly initializing the population for two parties. Initially, the populations of both parties are identical and consist of a single individual. Subsequently, each party employs one-bit mutation to iteratively search for non-dominated solutions, which are stored in their respective populations. Simultaneously, non-dominated solutions common to both parties are identified and stored. $\mathbf{z} \succ_m \mathbf{x}'$ in the algorithm indicates that $\mathbf{z}$ dominates $\mathbf{x}'$ on the objectives of party $m$, and $\mathbf{z} \succ_{m'} \mathbf{x}'$ in the algorithm indicates that $\mathbf{z}$ dominates $\mathbf{x}'$ on the objectives of party $m'$.

First, we prove the correctness of the algorithm \( \text{EMPMO}_{\text{simple}} \), demonstrating that it can identify the common non-dominated solutions of both parties in each iteration and, ultimately, find the common Pareto optimal solution.

\begin{lemma}\label{lemma_Phi}
	Consider an algorithm that iteratively constructs the set $\Phi$ and the populations $P_m,\, m \in \{1, \dots, M\} $ for each party through a sequence of mutation and selection steps, satisfying the following properties:
	
	\begin{enumerate}
		\item A new decision vector is added to $P_m$ if, and only if it is not weakly dominated by any member of $P_m$, and it is added to $\Phi$ if, and only if it is not weakly dominated by any member of $\Phi$ across all parties.
		\item A decision vector is removed from $P_m$ if and only if, a dominating vector is added to $P_m$, and it is removed from $\Phi$ if and only if, a dominating vector is added to $\Phi$ for at least one party.
	\end{enumerate}
	
	Then, the set $\Phi$ is the common non-dominated solution set of all parties, and $\Phi$ is the common Pareto optimal solution set if each population $P_m$ is the Pareto optimal solution set $PS_m$ for party $m$.
\end{lemma}

\begin{proof}
	First, assume a vector \(\mathbf{x}\) is added to the set \(\Phi\). By Property 1, \(\mathbf{x}\) is not weakly dominated by any element in \(\Phi\) across all parties, and by Property 2, any element dominated by \(\mathbf{x}\) is removed from \(\Phi\). Thus, \(\Phi\) contains only \(\mathbf{x}\) and non-dominated solutions with respect to each other. According to Algorithm 2, only when $x$ is a non-dominated solution in $P_m$, that is, a local non-dominated solution, can it be determined whether it can be added to $\Phi$. Therefore, these solutions are also not weakly dominated by any member of \(P_m\). Hence, every element \(\mathbf{x}\) in \(\Phi\) is a non-dominated solution in each party, i.e., \(\mathbf{x} \in P_m\) for all \(m\). Therefore, we have $\Phi \subseteq \bigcap_{m=1}^M P_m$.
	
	Then, assume there exists a vector \(\mathbf{x}'\) such that \(\mathbf{x}' \in \bigcap_{m=1}^M P_m\) but \(\mathbf{x}' \notin \Phi\). Then $x'$ is a non-dominated solution in each $P_m$. When $P_m=PS_m$, $x'$ is a global non-dominated solution. Since \(\mathbf{x}' \in P_m\), it is not weakly dominated by any vector in any population \(P_m\). However, if \(\mathbf{x}' \notin \Phi\), this means there exists a vector \(\mathbf{z}\) that weakly dominates \(\mathbf{x}'\) and has been added to \(\Phi\). By Property 1, if \(\mathbf{z}\) weakly dominates \(\mathbf{x}'\), then \(\mathbf{x}'\) cannot be a non-dominated solution for some party \(m\), leading to a contradiction. Thus, we have $\bigcap_{m=1}^M P_m \subseteq \Phi$.
	
	From steps 1 and 2, we conclude $\Phi = \bigcap_{m=1}^M P_m$. Finally, if each \(P_m = PS_m\), then $\Phi = \bigcap_{m=1}^M P_m$ represents the common Pareto set.
\end{proof}

The time complexity of \( \text{EMPMO}_{\text{simple}} \) for solving BPAOAZ primarily arises from independently searching for the complete Pareto optimal sets of AORZ and AOFZ. Finding their non-dominated solutions' intersection adds only a constant-time operation in each iteration. Therefore, the time complexity can be analyzed by referencing \cite{laumanns2004running}, which provides a framework for calculating the time complexity of SEMO when solving AORZ and AOFZ, involving two key lemmas, i.e., Lemmas \ref{lemma: GUBI} and \ref{lemma: GUBII}. It is important to note that the running time of the algorithms discussed in this paper is defined as the number of fitness evaluations required to include all Pareto optimal solutions in the population for the first time.

\begin{lemma}[General Upper Bound I \cite{laumanns2004running}]\label{lemma: GUBI}
	Consider an algorithm that iteratively updates a population \( P \) via a sequence of mutation and selection steps, with the following properties:
	\begin{enumerate}
		\item For each \( y \in F \setminus F^* \), the probability that the mutation operator, when applied to any \( \mathbf{x} \in X \) with \( f(\mathbf{x}) = y \), produces a dominating vector \( \mathbf{x}' \succ \mathbf{x} \) is at least \( p(y) > 0 \), where $F$ is the objective space and $F^*$ represents the PF.
		\item A new decision vector is added to \( P \) if and only if it is not weakly dominated by any member of \( P \).
		\item A decision vector is removed from \( P \) if and only if a dominating vector is added to \( P \).
	\end{enumerate}
	
	Then, the expected number of applications of the mutation operator to non-Pareto-optimal vectors is bounded above by \( \sum_{y \in F \setminus F^*} {p(y)}^{-1} \).	
\end{lemma}

%
%

\begin{lemma}[General Upper Bound II \cite{laumanns2004running}]\label{lemma: GUBII}
	Let the dominated part of the decision space \( X \setminus X^* \) be partitioned into \( k \) disjoint sets \( X_1, X_2, \ldots, X_k \) such that \( \bigcup_{i=1}^k X_i = X \setminus X^* \) and \( X_i \cap X_j = \emptyset \) for \( i \neq j \). Define the dominance relation on sets as:
	\[
	X_j \succ X_i \iff \forall (a, b) \in X_j \times X_i: a \succ b
	\]
	
	Let \( d(X_i) := \{ X_j \mid X_j \succ X_i \} \) denote the set of all sets that dominate \( X_i \). If the algorithm satisfies the same properties as in Lemma \ref{lemma: GUBI}, and if \( p(X_i) \) is a lower bound for the probability that a mutation applied to an individual \( \mathbf{x} \in X_i \) generates a new individual \( \mathbf{x}' \) in a dominating decision space subset, i.e.,
	\[
	0 < p(X_i) \leq \min_{\mathbf{x} \in X_i} \Pr(\mathbf{x}' \in d(X_i) \mid \mathbf{x} \in X_i),
	\]
	then the expected number of times the mutation operator is applied to non-Pareto-optimal decision vectors is bounded above by $ \sum_{i=1}^{k} {p(X_i)}^{-1}$.
\end{lemma}

%
%

Based on the three lemmas above, we can derive the expected running time of  \( \text{EMPMO}_{\text{simple}} \) for solving the BPAOAZ problem.

\begin{theorem}
	The expected running time of $\text{EMPMO}_{\text{simple}}$ applied to BPAOAZ is bounded by $O(3n^{2}\operatorname{log}n).$
\end{theorem}
\begin{proof}
	From Lemma 1, it follows that when $P_m$ is the Pareto set of party $m$, and $\Phi$ is the common Pareto set, the running time of Algorithm 2 is equivalent to the cumulative running time.
	
	We begin by calculating the running time to find the Pareto optimal solution set for party 1, whose objective function is \( \mathrm{AORZ}(\mathbf{x}) \). Without loss of generality, the process is divided into two stages: 1) from initialization to finding a single solution in the Pareto optimal set; 2) from the first Pareto optimal solution to discovering all solutions in the Pareto optimal set.
	
	First, non-Pareto optimal solutions are grouped as $X_{i,j} := \{ \mathbf{x} \in X \mid f(\mathbf{x}) = (j, n - i - j) \} $, where $ i, j \in \{0, \ldots, n/2\} $, $i$ represents the Hamming distance to the Pareto set, and $j$ is the number of ones in the second half of a solution. By Lemma 2, \( X_{i,j} \) forms a partition of the search space. The total number of mutations for non-Pareto-optimal search points is thus bounded by
	\[
	\sum_{j=0}^{n/2} \sum_{i=1}^{n/2} \frac{n}{i} = O(\frac{n^2}{2} \log n).
	\]
	
	Once a Pareto optimal solution is found, new Pareto optimal solutions can be obtained by flipping bits in the second half of the solution. By expanding from the found Pareto optimal solution toward $j=0$ and toward $j=\frac{n}{2}$, the complete PF can be found. Let $j^*$ be the $f_{11}$-value of the first Pareto optimal solution. Define integers $a_t \le b_t$ such that the population contains Pareto optimal search solutions with an $f_{11}$-value of $i$, for all $i \in \{a_t, \ldots , b_t\}$. Initially, $a_t = b_t = j^*$. 
	
	While $a_t > 0$, another Pareto optimal solution with $f_{11}$-value of $a_{t+1} = a_t-1$ can be obtained by selecting a Pareto optimal solution with $f_{11}$-value of $a_t$ and making an appropriate $1$-bit flip. The probability of that occurring is at least \( \frac{1}{\frac{n}{2}+1} \cdot \frac{a_t}{n} \). Summing up expected waiting time and then the time to successfully extend to $a_t=0$ is 
	\[ \sum_{a_t=1}^{j^*}\frac{n(\frac{n}{2}+1)}{a_t} = O(\frac{n^2}{2} \log j^*).  \]
	In the worst case, $j^*=\frac{n}{2}$, and the total expected waiting time is \( O(\frac{n^2}{2} \log n) \).
	
	While $b_t < \frac{n}{2}$, another Pareto optimal solution with $f_{11}$-value of $b_{t+1} = b_t + 1$ can be obtained by selecting a Pareto optimal solution with $f_{11}$-value of $b_t$ and making an appropriate $0$-bit flip. The probability of that occurring is at least \( \frac{1}{\frac{n}{2}+1} \cdot \frac{ \frac{n}{2} - b_t}{n} \). Summing up expected waiting time and then the time to successfully extend to $b_t = \frac{n}{2}$ is 
	\[\sum_{b_t=j^*}^{\frac{n}{2}-1}\frac{n(\frac{n}{2}+1)}{\frac{n}{2} - b_t} = O(\frac{n^2}{2} \log (\frac{n}{2}-j^*)). \]
	In the worst case, $j^*=0$, and the total expected waiting time is \( O(\frac{n^2}{2} \log n) \).

	
	So the total time until both goals have been reached is \( O(n^2 \log n) \). Combining the two stages, the expected time for solving \( \mathrm{AORZ} \) is \( O(\frac{3}{2}n^2\log n) \). By Lemma 1, the total running time of Algorithm 2 for solving \( \mathrm{BPAOAZ} \) is $O(3n^2\log n) $
\end{proof}

The $\text{EMPMO}_{\text{simple}}$ is an idealized algorithm for solving MPMOPs, relying on the assumption that sufficient prior knowledge is available to ensure the algorithm operates correctly. To address this limitation, we propose a more general framework, referred to as  the random evolutionary multi-party multi-objective optimizer ($\text{EMPMO}_{\text{random}}$), as shown in Algorithm \ref{alg:MPCOEAr}. 
	Unlike $\text{EMPMO}_{\text{simple}}$, $\text{EMPMO}_{\text{random}}$ maintains a single population during its execution period, preserving the negotiation solutions of both parties. In each iteration, a new solution is generated by one party using one-bit mutation. Subsequently, all dominated solutions from that party are removed from the population.

	\begin{algorithm}[!h]
		\caption{Random Evolutionary Multi-party Multi-objective Optimizer ($\text{EMPMO}_{\text{random}}$) } 
		\label{alg:MPCOEAr}
		\begin{algorithmic}[1]
			\State Randomly select an individual $\mathbf{x} \in X$ ;
			\State $P\leftarrow \{ \mathbf{x}\}$;
			\While {termination condition is not met}
			\State Randomly select an individual $\mathbf{x} \in P$;
			\State Randomly select a party $m$;
			\State Apply one-bit mutation to $\mathbf{x}$ to generate $\mathbf{x}'$;
			\If{$\nexists \mathbf{z} \in P \: \mathrm{satisfying}  \:(\mathbf{z} \succ_m \mathbf{x}'\lor F_m(\mathbf{z})=F_m(\mathbf{x}'))$}
			\State $P\leftarrow(P \setminus \{\mathbf{z}\in P \mid \mathbf{x}'\succ_m \mathbf{z}\})\cup\{\mathbf{x}'\}$;
			\EndIf
			\State $P\leftarrow\{\mathbf{z}\in P\mid\nexists \mathbf{z}^{\prime}\in P \setminus \{\mathbf{z}\}  ,\mathbf{z}^{\prime}\succeq_m \mathbf{z}\}$; 
			\EndWhile
		\end{algorithmic}
	\end{algorithm}
	
	For BPAOAZ, we assume that each party is selected with probabilities $\varphi$ and $1-\varphi$, respectively, where $\varphi$ is the probability of choosing the first party. Based on this assumption, we can derive the expected running time of $\text{EMPMO}_{\text{random}}$ for solving the BPAOAZ problem.
	
	\begin{theorem}\label{the:random}
		The expected running time of $\text{EMPMO}_{\text{random}}$ applied to BPAOAZ is bounded by $O\left(\left(\frac{1}{\varphi} + \frac{1}{1-\varphi} \right) \frac{n^{2}}{2} \log n \right).$
	\end{theorem}
	
	\begin{proof}
		The evolutionary process is divided into two stages. The first stage aims to find a Pareto-optimal solution in $\Phi_1 \cup \Phi_2$, and the second stage seeks a common Pareto-optimal solution starting from the solution found in the first stage.
		
		Without loss of generality, assume a Pareto-optimal solution in $\Phi_2$ is obtained at the end of the first stage, which is then evolved to obtain the common Pareto set $\Phi$ during the second stage. From the analysis of AOFZ, the objective space of party 2 is divided into $1 + \frac{n}{2}$ subsets $F_{2, j}$, where $F_{2, \frac{n}{2}} = F_2^*$. Each subset $F_{2, j}$ contains $1 + \frac{n}{2}$ objective vectors. According to the evolutionary rules of Algorithm \ref{alg:MPCOEAr}, in iteration $t$, let $m_s^t$ represent the direction of selection and $m_m^t$ represent the direction of actual mutation, which can make $\mathbf{x}'\succ_{m_m^t} \mathbf{x}$. If $m_s^t\neq m_m^t$, both solutions $\mathbf{x}$ and $\mathbf{x}'$ are non-dominated and retained in the population. Conversely, if $m_s^t=m_m^t$, $\mathbf{x}$ will be removed from the population. Therefore, when $m_s^t=m_m^t$, we consider this to be a successful mutation.

		Based on the size of $F_{2, j}$, we know that the probability of selecting a solution from $F_{2, j}$ for mutation is at least $\frac{1}{\frac{n}{2} + 1}$. During mutation, the probability of flipping a $0$ in the second half of a solution is $\frac{\frac{n}{2} - j}{n}$, and the probability of evolving to party 2 in one step is at least
		\(
		(1-\varphi) \cdot \frac{\frac{n}{2} - j}{n} \cdot \frac{1}{\frac{n}{2} + 1}.
		\)
		Because $X_2^* = \{a 1^{\frac{n}{2}} \mid a \in \{0,1\}^{\frac{n}{2}}\}$, $\frac{n}{2}$ successful flips suffice to find a Pareto-optimal solution in $\Phi_2$. Thus, the expected running time for the first stage is
		\[
		\sum_{j=0}^{\frac{n}{2}-1} \frac{n \left( \frac{n}{2} + 1 \right)}{\left( 1-\varphi \right) \left( \frac{n}{2} - j \right)} = O(\frac{1}{1-\varphi} \frac{n^2}{2} \log n).
		\]
		
		In the second stage, starting from the Pareto-optimal solution in $\Phi_2$, the algorithm seeks a common Pareto-optimal solution by evolving on party 1. The probability of selecting this solution from the population and successful mutation for one step is $\frac{\varphi k}{n \cdot (n/2+1)}$, where $k$ is the number of $0$ in the first half of the solution. Then the expected time to find the common Pareto set $\Phi$ is
		\[
		\sum_{k=1}^{\frac{n}{2}} \frac{n \cdot (\frac{n}{2} + 1) }{\varphi k} = O(\frac{1}{\varphi} \frac{n^2}{2} \log n).
		\]
		
		Combining both stages, the total expected running time is  $O\left(\left(\frac{1}{\varphi} + \frac{1}{1-\varphi} \right) \frac{n^2}{2} \log n \right)$.
	\end{proof}
	
	 $\text{EMPMO}_{\text{random}}$ , each mutation is controlled by one party at a time. This party is assumed to have incomplete information, specifically the Pareto dominance relationships of solutions relevant to itself, without access to global information. This decision-making model reflects the concept of bounded rationality, as described in game theory, where agents operate with partial knowledge and limited reasoning capabilities \cite{xu2022decision}. However, inspired by the behavior of fully rational agents, we aim to make optimal decisions in each round, maximizing the decision payoff based on complete information. Therefore, we introduce the concept of multi-party payoff to represent the various possible results in the context of the MPMOPs, such as improvements on both parties, deterioration on both parties, improvement on one party and deterioration on the other party, or no change at all.

	
	\begin{definition}[Multi-party Payoff]
		Let candidate solutions with the same objective values be grouped into a set \( X_i \). The payoff \( \pi_m(X_i, X_j) \) represents the degree of evolution or degeneration for party \( m \) when transitioning from set \( X_i \) to set \( X_j \). The multi-party payoff \( \pi(X_i, X_j) \) on MPMOPs for the transition from \( X_i \) to \( X_j \) is given by:
		\[
		\pi(X_i, X_j) = \sum_{m=1}^{M} \pi_m(X_i, X_j).
		\]
	\end{definition}
	
	In the specific implementation of this paper, we define $\pi_m(X_i, X_j)$ as the weighted sum of the differences in objective values. Specifically,
	$$\begin{aligned}
	\pi_m(X_{i},X_{j})=
	&\begin{cases}
		1, & \text{if } \, \forall k,f_{mk}(X_{j})\geq f_{mk}(X_{i})\\
		-1, & \text{if } \, \forall k,f_{mk}(X_{j}) \leq f_{mk}(X_{i})\\
		0, & \text{otherwise}
	\end{cases}
	\end{aligned}$$
	where $k \in \{1,...,k_m\}$ represents the objective, and $m \in \{1,2\}$ denotes party $m$.

	Consider modeling the evolutionary process using a Markov chain. The corresponding state transition diagram describes the transition probabilities between all states, from the initial state to the target state. Suppose that each edge in the state transition diagram is assigned an additional weight, representing the payoff from transitioning from one state to another. In this case, the payoff $\pi(X_i, X_j)$ from the initial state to the target state should be non-negative. Therefore, we can draw the following conclusion.
	
	\begin{lemma} \label{lemma: cycle}
		If the multi-party payoff of each transition between two states, $X_i$ and $X_j$, is greater than zero, and it is assumed that transitions only occur in the direction of positive payoffs, then no cycle exists between $X_i$ and $X_j$ in the Markov chain.
	\end{lemma}
\begin{proof}
	Assume, for the sake of contradiction, that there exists a cycle between states $X_i$ and $X_j$. This implies the existence of a path that starts at $X_i$, passes through one or more states, and returns to $X_i$. Since each transition must occur in the direction of positive payoffs, the total payoff along this cycle must be greater than zero. This contradicts the assumption that the total payoff of any cycle is zero. Therefore, no cycle can exist between $X_i$ and $X_j$.
\end{proof}

	According to Lemma \ref{lemma: cycle}, as long as the payoff of each evolutionary step is greater than zero, we can reach the target state from the initial state, since no cycles exist in the evolution chain, preventing the degradation of the population on MPMOPs. Based on this, we propose a payoff-based evolutionary multi-party multi-objective optimizer ($\text{EMPMO}_{\text{payoff}}$), as outlined in Algorithm \ref{alg:MPCOEAp}. First, the population is initialized with a single individual. The population then undergoes continuous evolution in a loop. In each iteration, a solution $\mathbf{x}$ is randomly selected from the population and mutated to generate a new solution $\mathbf{x}'$. 
	If the payoff from $\mathbf{x}$ to $\mathbf{x}'$ is greater than zero, it means that $\mathbf{x}$ has evolved towards at least one of the two parties, resulting in $\mathbf{x}'$ that is closer to the common optimal solution.  
	Then $\mathbf{x}'$ is added to the population, and $\mathbf{x}$ is removed. This selection strategy ensures that the population progresses at each step, evolving toward the global Pareto optimal solution.
	
	\begin{algorithm}[!h]
		\caption{  Payoff-based Evolutionary Multi-party Multi-objective Optimizer ($\text{EMPMO}_{\text{payoff}}$) }
		\label{alg:MPCOEAp}
		\begin{algorithmic}[1]
			\State Randomly select an individual $\mathbf{x} \in X$ ;
			\State $P\leftarrow \{ \mathbf{x}\}$;
			\While {termination condition is not met}
			\State Randomly select an individual $\mathbf{x} \in P$;
			\State Apply one-bit mutation to $\mathbf{x}$ to generate $\mathbf{x}'$ ;
			\If{$\pi_{\mathbf{x},\mathbf{x}'} > 0$}
			\State $P\leftarrow(P \setminus \{\mathbf{x}\})\cup\{\mathbf{x}'\}$;
			\EndIf
			\EndWhile
		\end{algorithmic}
	\end{algorithm}
	
	The $\text{EMPMO}_{\text{payoff}}$ could achieve more efficient successful mutations when solving the BPAOAZ problem. Consequently, we can derive its expected running time for solving the BPAOAZ problem, as follows:
	
	\begin{theorem}
		The expected running time of $\text{EMPMO}_{\text{payoff}}$ applied to BPAOAZ is bounded by $O(n\log n)$.
	\end{theorem}
	\begin{proof}
		Given the problem setting of BPAOAZ and the mutation strategy, where each successful mutation results in a positive multi-party payoff for all possible state transitions in the evolution process, there is only an objective vector in the population throughout the evolution. Let $i$ represent the number of zeros in the solution. Since the common Pareto optimal solution is $\Phi = {1^n}$, it takes at most $n$ steps to increase the number of zeros from $i$ to $n$, reaching the common Pareto optimal solution. The probability of a solution evolving in each step is $\frac{i}{n}$. Therefore, the total expected running time is given by $\sum_{i=1}^{n} \frac{n}{i}$, which simplifies to $O(n \log n)$.
	\end{proof}
	
	 \( \text{EMPMO}_{\text{payoff}} \) can be regarded as a special case of  \( \text{EMPMO}_{\text{random}} \), where the correct party is selected for evolution in each round rather than choosing randomly. Given the specific structure of the problem, where the common Pareto set \( \Phi \) contains only a single element, the algorithm achieves a runtime bound of \( O(n \log n) \). Specifically, the runtime of  \( \text{EMPMO}_{\text{payoff}} \) is equivalent to the time required to compute the transition from \( 0^n \) to \( 1^n \) via one-bit mutation. Therefore, its time complexity establishes the lower bound \( \Omega(n \log n) \) for \( \text{EMPMO}_{\text{random}} \) in solving the BPAOAZ problem.
	
	Although \( \text{EMPMO}_{\text{payoff}} \)  represents an idealized case, it can be adapted to real-world problems by defining payoffs based on indicators \cite{falcon2020indicator} and estimating each party's payoff using historical information \cite{xu2022difficulty}.
	
	\subsection{Comparison Between Multi-party Multi-objective Optimization and Multi-objective Optimization} \label{four}
	
	Excluding the multi-party attributes, the BPAOAZ problem can be conventionally treated as a general MOP, defined as follows:
	
	\begin{definition}[AOAZ]
		The pseudo-Boolean function AOAZ, denoted as \(\mathrm{AOAZ} : \{0,1\}^n \to \mathbb{N}^4\), is defined by
		\[
		\mathrm{AOAZ}(\mathbf{x}) = \left(f_{11}(\mathbf{x}), f_{12}(\mathbf{x}), f_{21}(\mathbf{x}), f_{22}(\mathbf{x})\right),
		\]
		where $f_{11}(\mathbf{x})$, $f_{12}(\mathbf{x})$, $f_{21}(\mathbf{x})$, and $f_{22}(\mathbf{x})$ are defined in Definition \ref{def: BPAOAZ}. Here, $n = 2a$, and $a \in \mathbb{N}$.
	\end{definition}
	
	The PF $F^{*}$ of AOAZ is the union of $F_1^{*}$ of AOFZ and $F_2^{*}$ of AORZ , with cardinality $|F^{*}|=n+1$. The Pareto set $X^{*}=\{1^{\frac{n}{2}}a,a \in \{0,1\}^{\frac{n}{2}}\} \cup  \{a1^{\frac{n}{2}},a \in \{0,1\}^{\frac{n}{2}}\}$ has cardinality $|X^{*}| = 2^{\frac{n}{2}+1}-1$. Moreover, the Pareto set of AOAZ is a superset of the common Pareto set of BPAOAZ.
	
	Once BPAOAZ is treated as a general MOP, the time complexity for finding the complete Pareto set from a single Pareto optimal solution increases significantly. Additionally, the probability of successful evolution from any non-Pareto optimal solution decreases. Example \ref{example1} illustrates this scenario.
	
	\begin{example} \label{example1}
		Consider the following example: Let \(\mathrm{BPAOAZ} : \{0,1\}^{100} \to \{\mathbb{N}^2, \mathbb{N}^2\}\) with problem size \(n = 100\), and \(\mathrm{AOAZ} : \{0,1\}^{100} \to \mathbb{N}^4\) with the same problem size. Let \(\mathbf{x}_1\) and \(\mathbf{x}_2\) be two individuals such that
		\[
		\begin{aligned}
			\mathrm{BPAOAZ}(\mathbf{x}_1) &= \left((36, 44), (56, 30)\right), \\
			\mathrm{BPAOAZ}(\mathbf{x}_2) &= \left((35, 40), (60, 25)\right).
		\end{aligned}
		\]
		
		In the \(\mathrm{BPAOAZ}\), \(\mathbf{x}_1\) dominates \(\mathbf{x}_2\) on one party, while neither dominates the other on the second party. Therefore, \( \text{EMPMO} \) can evolve successfully on the first party. However, from the perspective of the \(\mathrm{AOAZ}\), we have:
		\[
		\begin{aligned}
			\mathrm{AOAZ}(\mathbf{x}_1) &= \left(36, 44, 56, 30\right), \\
			\mathrm{AOAZ}(\mathbf{x}_2) &= \left(35, 40, 60, 25\right).
		\end{aligned}
		\]
		
		In this case, \(\mathbf{x}_1\) and \(\mathbf{x}_2\) are mutually non-dominating.
	\end{example}
	
	We can further derive the expected running time of using the multi-objective optimization algorithm SEMO to solve AOAZ as follows:
	
	\begin{theorem}
		The expected running time of SEMO applied to AOAZ is bounded by $O(\frac{5}{2}n^{2}\log n)$.
	\end{theorem}
	\begin{proof}
		Let \(s\) denote the Hamming distance between the current solution and the Pareto solution set, representing the minimum number of zeros in either the first or second half of the solution. Let each objective vector constitutes its own subset and the probability of improvement for each set is \(\frac{s}{n}\).
		
		The evolution process is divided into two stages. The first stage aims to find a Pareto optimal solution, and the second stage finds the complete Pareto optimal solution set.
		
		For the first stage, we group solutions with the same objective vector together and represent them using the Hamming distance $s$ and $i$, the number of 1s in the second half of the solution. By Lemma 2, the expected number of mutations required to improve the solution is $
		\sum_{i=1}^{n/2}\sum_{s=1}^{i}\frac{n}{s} \leq \sum_{i=1}^{n/2}n\log i =n\cdot\log[(\frac{n}{2})!]
		$. According to Stirling’s formula, we have $(\frac{n}{2})!\leq e^{\frac{1}{12}}\cdot\sqrt{\pi n}(\frac{n}{2e})^{\frac{n}{2}}$, and then $\sum_{i=1}^{n/2}\sum_{s=1}^{i}\frac{n}{s} =O(n\cdot\log(n^{\frac{n+1}{2}}))=O(\frac{n^{2}}{2}\log n)$.
		
		For the second stage, we start from the Pareto optimal solution found and look for the complete Pareto optimal solution set. Without loss of generality, we assume that the solution found is $0^{\frac{n}{2}}1^{\frac{n}{2}}$. First, we need to flip 0-bit in the first half to get $1^{n}$, and then flip 1-bit in the second half to get the complete Pareto optimal solution set. The maximum number of iterations required is 
		\[
		\sum_{i=1}^{\frac{n}{2}}\frac{1}{\frac{1}{n+1}\cdot\frac{i}{n}}+\sum_{j=1}^{\frac{n}{2}}\frac{1}{\frac{1}{n+1}\cdot\frac{j}{n}}=O(2n^{2}\log n).
		\]
		Therefore, the total expected running time is $O(\frac{5}{2}n^{2}\log n)$.
		
	\end{proof}
	
	In the proofs of Theorem 1 to Theorem 4, we retain the constant terms. Comparing the running time of solving the AOAZ problem and the BPAOAZ problem, we can intuitively find that the EMPMO consistently requires less time. The theoretical upper bound of running time of \( \text{EMPMO}_{\text{simple}} \) is longer than that of SEMO because the second stage analysis of \( \text{EMPMO}_{\text{simple}} \), when expanding from a Pareto optimal solution to PS, considers the worst-case scenario in both directions. Based on the experimental results in Section~\ref{sec:experiments}, it's clear that \( \text{EMPMO}_{\text{simple}} \) performs better than SEMO. Specifically, the second stage analysis process of SEMO is close to the complete analysis process of \( \text{EMPMO}_{\text{random}} \), but the selection probability is smaller. The calculation method of the second stage of SEMO is close to the analysis process of \( \text{EMPMO}_{\text{simple}} \), but its selection probability is still smaller.
	This is because SEMO can't solve BPAOAZ problem directly and it can't achieve good results. SEMO is a standard multi-objective optimization algorithm, whose objective is to solve the complete Pareto optimal solution set. It cannot directly solve multi-party multi-objective optimization problems to obtain the common solution set. A multi-party multi-objective optimization problem can be formulated as a multi-objective optimization problem by introducing suitable trade-offs. However, this will lead to the well-known issue of large PF. This indicates that when the problem inherently involves multiple decision-makers collaboratively making decisions, formulating it as a MPMOP is more efficient than representing it as a standard MOP. Without dividing multiple decision-makers, the problem is multi-objective. Once divided into multiple decision-makers, the solution set of the problem will become smaller, and may even degenerate into a single-objective problem.
	
	Moreover, obtaining the complete Pareto set for a MOP does not guarantee an optimal decision, as each decision-maker focuses solely on their own objectives. For example, in the BPAOAZ problem, only a single common Pareto optimal solution satisfies the requirements of both parties. Therefore, solving the corresponding bi-party multi-objective optimization problem is not only more efficient but also a more rigorous approach.

\section{Runtime Analysis of Evolutionary Algorithms on Bi-party Multi-objective Shortest Path Problem}
\label{sec: theoretical analysis2}
We analyze the multi-objective shortest path (MOSP) problem as a case study to compare the running times of MPMOEAs and traditional MOEAs in solving NP-hard problems \cite{serafini1987some, horoba2010exploring}. MOSP is representative of real-world applications \cite{zajac2021objectives, chen2023heuristic}, such as UAV path planning, where trade-offs between safety and efficiency must be resolved in multi-party scenarios \cite{10463192}. This section introduces the bi-party multi-objective shortest path problem (BPMOSP) and derives a baseline algorithm, the simple evolutionary bi-party multi-objective optimizer (\( \text{EMPMO}^{\text{SP}}_{\text{simple}} \)) , based on the diversity-maintaining evolutionary multi-objective optimizer (DEMO) \cite{horoba2009analysis, horoba2010exploring}. We analyze the runtime of this algorithm for solving BPMOSP. Building on this, we propose the consensus-based evolutionary multi-party multi-objective optimizer (\( \text{EMPMO}^{\text{SP}}_{\text{cons}} \)), which incorporates only candidate solutions achieving consensus between the two parties into the population for updates. A detailed runtime analysis is provided for the proposed algorithm, and its performance is compared with that of traditional evolutionary multi-objective optimizers.

\subsection{Bi-party Multi-objective Shortest Path Problem}
We consider the bi-party multi-objective single-source shortest path problem, where the objective is to find the common Pareto shortest path set from a fixed source vertex to each of the remaining target vertices. The formal definition is provided as follows. Assume that all other vertices are reachable from the source vertex directly or indirectly.

\begin{definition}
	Given a directed weighted graph \( G = (V, E, W) \), where \( V \) is the set of vertices and \( E\subseteq\{(u,v)\in V^{2}\mid v\neq u\} \) is the set of directed edges, and $W: E\to \mathbb{R}^{k_1} \times \mathbb{R}^{k_2}$ is the weight function that assigns each edge $e = (u, v) \in E$ a vector of weights corresponding to the objectives of two parties, the bi-party multi-objective single-source shortest path problem (BPMOSP) is defined as:
	\[
	\min_{\mathbf{x} \in \text{Paths}(s, v)} \mathcal{F}(\mathbf{x}) = \left(F_{1}(\mathbf{x}), F_{2}(\mathbf{x})\right),
	\]
	where $\mathbf{x} = \{v_0, v_1, \dots, v_l\}$ represents a path from a fixed source vertex $s = v_0$ to any other vertex $v = v_l \in V$, \( \mathrm{Paths}(s, v) \) denotes the set of all paths from the source \( s \) to target \( v \). The multi-objective functions $F_m(\mathbf{x})$ for each party $m \in \{1, 2\}$, each containing $k_m$ objectives, are given by:
	\[
	F_m(\mathbf{x}) = \left(f_{m1}(\mathbf{x}), \dots, f_{mk_m}(\mathbf{x})\right),
	\]
	where $f_{mk}$ is the sum of the weights along the path $\{v_0,...,v_l\}$ corresponding to the $k$-th objective of party $m$:
	\[
	f_{mk}(\mathbf{x}) = \sum_{j=1}^{l} w_{mk}(v_{j-1}, v_j),
	\]
	with $w_{mk}(v_{j-1}, v_j)$ representing the weight of the edge $(v_{j-1}, v_j)$ corresponding to $k$-th weight of party $m$. The path length $l$ satisfies $l \le n-1$, where $n = |V|$ is the number of vertices in the graph $G$.
\end{definition}

To facilitate theoretical analysis, we define the following notation:
\[
\left\{
\begin{aligned}
	w_m^{\max} &= \max_{k \in \{1, \dots, k_m\}} \left( \max_{e \in E} w_{mk}(e) \right), \\
	w_m^{\min} &= \min_{k \in \{1, \dots, k_m\}} \left( \min_{e \in E} w_{mk}(e) \right),\\
\end{aligned}
\right.
\]
which represent the maximum and minimum weights assigned by party \( m \), respectively. Additionally, we define
\[
\left\{
\begin{aligned}
	w^{\max} &= \max_{m\in \{1, 2\}} w_m^{\max}, \\
	w^{\min} &= \min_{m\in \{1, 2\}} w_m^{\min}, \\
\end{aligned}
\right.
\]
to denote the maximum and minimum weights across all parties.

In this problem, all edge weights are assumed to be positive. We set \( w^{\min} \geq 1 \), which can be achieved by normalizing all weights by dividing them by \( w^{\min} \). Without loss of generality, among the \( n \) vertices, we designate vertex \( 1 \) as the source vertex. Thus, the goal is to find the common shortest paths from \( s = 1 \) to the remaining \( n-1 \) vertices. According to the problem setting, the objective function satisfies \( 1 \leq f_{mk}(\mathbf{x}) \leq (n-1)w_{m}^{\text{max}} \). Then, the dominance relationship between candidate solutions with the same target vertex in the BPMOSP aligns with the definition provided in Definition~\ref{def:Domination}.


The time required to identify the PF for most MOPs is well known to grow exponentially with the problem size \cite{bokler2017multiobjective}. For NP-hard MOPs, approximating the PF is an effective strategy \cite{laumanns2002combining}, and this approach can also be extended to NP-hard MPMOPs. Let \( (1 + \varepsilon) \) represent the approximation ratio.

\begin{definition}[\( \varepsilon \)-domination]\label{def:epsilon}
	Let \( \mathbf{x} \) and \( \mathbf{x}' \) denote paths with target vertices \( v_l \) and \( v_{l'} \), respectively. Path \( \mathbf{x}' \) weakly \( \varepsilon \)-dominates path \( \mathbf{x} \) for party \( m \), denoted as \( F_m(\mathbf{x}') \succeq_{1+\varepsilon} F_m(\mathbf{x}) \), if and only if:
	\[
	\forall k \in \{1, \dots, k_m\}, \, f_{mk}(\mathbf{x}') \leq (1+\varepsilon) \cdot f_{mk}(\mathbf{x}) \quad \text{and} \quad v_l = v_{l'}.
	\]
	
	where $\varepsilon>0$. If \( F_m(\mathbf{x}') \succeq_{1+\varepsilon} F_m(\mathbf{x}) \) and \( F_m(\mathbf{x}') \neq F_m(\mathbf{x}) \), then \( \mathbf{x}' \) \( \varepsilon \)-dominates \( \mathbf{x} \), denoted as \( F_m(\mathbf{x}') \succ_{1+\varepsilon} F_m(\mathbf{x}) \). Paths with different target vertices are considered as incomparable.
\end{definition}

In the BPMOSP, the concept of common Pareto optimality is consistent with Definition \ref{def:CommonPareto} and can be extended to the case of \( \varepsilon \)-dominance as defined in Definition \ref{def:epsilon}. Assuming the existence of at least one path \( \mathbf{x} \) that satisfies the definition of common optimality, this paper focuses on MPMOPs with common solutions.

\subsection{Runtime Analysis of Simple Evolutionary Multi-party Multi-objective Optimizer}

Intuitively, Algorithm \ref{alg:MPCOEAs} can also be applied to solve the BPMOSP. Assuming the existence of a common Pareto solution, the algorithm seeks to identify the complete Pareto set for each party individually, with their intersection yielding the desired common Pareto set. However, for NP-hard problems, obtaining the exact PF is computationally expensive. Therefore, it is often more practical to pursue solutions that satisfy a \((1 + \varepsilon)\)-approximation ratio, which reduces the time complexity to a polynomial level. As proposed by Horoba \cite{horoba2010exploring}, a runtime analysis framework for evolutionary algorithms can be used to evaluate their \( (1 + \varepsilon) \)-approximation performance on the multi-objective shortest path problem (MOSP). This framework employs the concept of a box index \cite{laumanns2002combining} to partition the objective space. 

\begin{definition}[Box Index \cite{laumanns2002combining}] \label{box}
	For the box size $r>1$, the box index of the objective vector $F_m(\mathbf{x})$, corresponding to the path $\mathbf{x}$, is defined as:
	$$b_r(F_{m}(\mathbf{x}))=(\lfloor\log_{r}(f_{m1}(\mathbf{x}))\rfloor,\cdots,\lfloor\log_{r}(f_{mk_{m}}(\mathbf{x}))\rfloor).$$
\end{definition}

Each box can accommodate at most one individual in the population, effectively controlling the population size through box division and the dominance relationship between boxes. According to the properties of the MOSP, box indices are comparable only when the target vertices of the paths are identical. If the target vertices differ, the box indices are considered incomparable. Formally, given two paths $\mathbf{x}$ and $\mathbf{x}'$ with the identical target vertices, \( b_r(F_m(\mathbf{x}')) \succeq b_r(F_m(\mathbf{x})) \) if and only if
\[
\forall k \in \{1, \dots, k_m\}, \, \lfloor \log_r(f_{mk}(\mathbf{x}')) \rfloor \leq \lfloor \log_r(f_{mi}(\mathbf{x})) \rfloor,
\]
and \( b_r(F_m(\mathbf{x}')) \succ b_r(F_m(\mathbf{x})) \) if and only if
\[
\begin{aligned}
	\forall k \in \{1, \dots, k_m\}, & \lfloor \log_r(f_{mk}(\mathbf{x}')) \rfloor \leq \lfloor \log_r(f_{mk}(\mathbf{x})) \rfloor, \\
	\exists k \in \{1, \dots, k_m\}, & \lfloor \log_r(f_{mk}(\mathbf{x}')) \rfloor < \lfloor \log_r(f_{mk}(\mathbf{x})) \rfloor.
\end{aligned}
\]

Using the runtime analysis framework proposed by Horoba \cite{horoba2010exploring}, we independently search for \( (1 + \varepsilon) \)-approximation solutions for both parties. However, unlike the conclusion in Section \ref{22}, after introducing approximation analysis, there is no guarantee that the algorithm can find a common Pareto set satisfying the \( (1 + \varepsilon) \)-approximation ratio.


\begin{lemma}\label{theo:empty}
	Let \( \Phi_1 \) and \( \Phi_2 \) represent the Pareto-optimal solution sets of the BPMOSP for two parties, respectively, and assume there exists a common Pareto solution \( \mathbf{x}^* \in \Phi_1 \cap \Phi_2 \). Suppose \( \Phi_1^{\varepsilon_1} \) and \( \Phi_2^{\varepsilon_2} \) are the Pareto solution sets with approximation ratio $(1+\varepsilon_1)$ and $(1+\varepsilon_2)$, respectively, obtained via EAs. Then, there exists a non-zero probability that the intersection of these approximate sets is empty, i.e., $Pr\left( \Phi_1^{\varepsilon_1} \cap \Phi_2^{\varepsilon_2} = \emptyset \right) > 0 $.
\end{lemma}

\begin{proof}
	Assume that a common Pareto solution \( \mathbf{x}^* \) exists. According to Definition \ref{box}, the box indices for party 1 and party 2 corresponding to \( \mathbf{x}^* \) are \( b_{r_1}(F_1(\mathbf{x}^*)) \) and \( b_{r_2}(F_2(\mathbf{x}^*)) \), where \( r_1 = 1 + \varepsilon_1 \) and \( r_2 = 1 + \varepsilon_2 \) represent the approximation ratios for each party. Since evolutionary algorithms (EAs) are employed to approximate the Pareto sets, there exist solutions \( \mathbf{x}_1 \) and \( \mathbf{x}_2 \) such that:
	\[
	b_{r_1}(F_1(\mathbf{x}_1)) = b_{r_1}(F_1(\mathbf{x}^*)), \quad b_{r_2}(F_2(\mathbf{x}_2)) = b_{r_2}(F_2(\mathbf{x}^*)).
	\]
	
	Given that each approximation box can contain at most one solution, it follows that \( \Phi_1^{\varepsilon_1} = \{\mathbf{x}_1\} \) and \( \Phi_2^{\varepsilon_2} = \{\mathbf{x}_2\} \). Without loss of generality, we aim to prove that at least one scenario exists where \( \mathbf{x}_1 \neq \mathbf{x}_2 \).
	
	To illustrate this, consider a weighted directed graph \( G = (V, E) \), as depicted in Fig. \ref{fig:np}, where \( V = \{1, 2, 3, 4, 5\} \). Let \( \varepsilon_1 = \varepsilon_2 = 1 \), resulting in a final box size of 2 for both parties. Table \ref{tab: G} enumerates all possible paths from the source vertex \( 1 \) to the target vertex \( 5 \) within the solution space, along with their corresponding objective values for the two parties. Boldface values highlight objective values that are not dominated by other solutions.
	
	\begin{figure}[h]
		\centering
		\small
		\includegraphics[width=0.9\columnwidth]{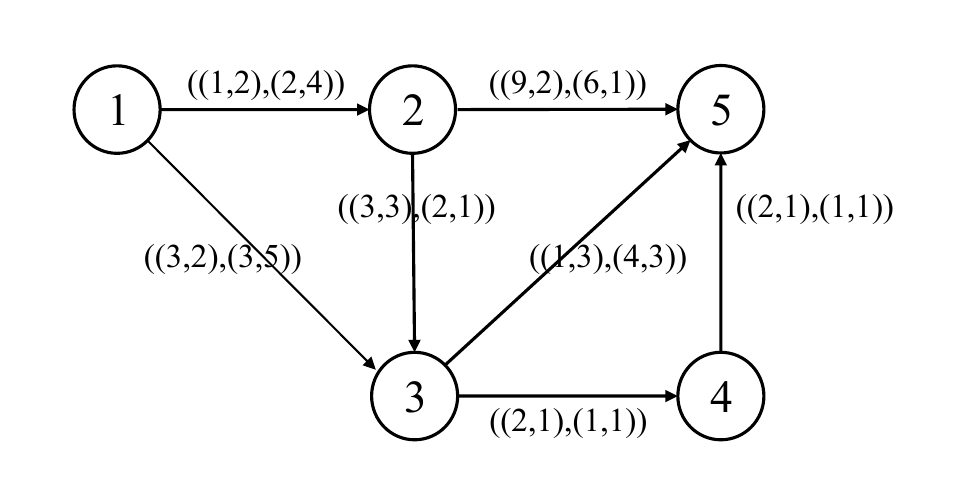}
		\caption{The weighted directed graph $G$.}
		\label{fig:np}
	\end{figure}
	
	\begin{table}[h]
		\centering
		\caption{All possible solutions from vertex $1$ to vertex $5$ and their respective objective values on the two parties.}
		\label{tab: G}
		\begin{tabular}{ccc}
			\toprule
			Paths    & Objective value on party $1$    & Objective value on party $2$               \\ \midrule
			$(1,2,5)$      & $(10,4)$           & $\mathbf{(8,5)}$            \\
			$(1,2,3,5)$    & $(5,8)$           & $(8,8)$        \\
			$(1,2,3,4,5)$  & $(8,7)$           & $(6,7)$        \\
			$(1,3,5)$      & $\mathbf{(4,5)}$  & $(7,8)$        \\
			$\mathbf{(1,3,4,5)}$    & $\mathbf{(7,4)}$  & $\mathbf{(5,7)}$        \\ \bottomrule
		\end{tabular}
	\end{table}
	
	For party 1, the Pareto set from vertex 1 to vertex 5 is \(\Phi_1 = \{(1, 3, 5), (1, 3, 4, 5)\} \). For party 2, the Pareto set is \( \Phi_2 = \{(1, 2, 5), (1, 3, 4, 5)\} \). Therefore, $\mathbf{x}^* =  (1, 3, 4, 5) $ is the common Pareto-optimal solution. 
	
	Now, consider the approximate solutions obtained via the evolutionary algorithm. Let the path sets for the two parties be \( P_1 = \{ (1, 3, 5), (1, 3, 4, 5) \} \) for party 1, and \( P_2 = \{ (1, 2, 5) \} \) for party 2. The corresponding box indices for party 1 and party 2 are $b_{r_1}(F_1(P_1)) = \{(2, 2)\}$ and $b_{r_2}(F_2(P_2)) = \{(3, 2)\}$.
	
	For the solutions in \( P_1 \), the box indices for party 2 are $\{(2, 3),(2, 2)\}$, which are not included in \( b_{r_2}(F_2(P_2)) \). This implies that the intersection of \( \Phi_1^{\varepsilon_1} \) and \( \Phi_2^{\varepsilon_2} \) is empty. Thus, despite the existence of a common Pareto solution \( (1, 3, 4, 5) \), we cannot find it in the approximate solution sets at this time.
	
	Therefore, the probability that the intersection of the approximate solution sets is empty is non-zero:
	\(
	\text{Pr} \left( \Phi_1^{\varepsilon_1} \cap \Phi_2^{\varepsilon_2} = \emptyset \right) > 0.
	\)
\end{proof}

However, as an approximation algorithm, this issue can be mitigated by relaxing the approximation ratio of one party. Taking Fig. \ref{fig:np} as an example, we relax the box size of party 2, represented as \(r_2^{\prime} = 1 + \varepsilon_2^{\prime}\), until a solution is found within the common non-dominated box, setting \(\varepsilon_2^{\prime} = 2\) expands the box size for party 2, resulting in \( b_{r_2^{\prime}}(F_2(P_2)) = \{(1,1)\}\). The box indices of the solutions in \(P_1\) for party 2 are \((1,1)\). At this point, the path \((1,3,5)\) and \((1,3,4,5)\) satisfies the required conditions and is output as the final solutions from vertex 1 to vertex 5. 

The pipeline of \(\text{EMPMO}^{\text{SP}}_{\text{simple}}\) for the MPMOSP is outlined in Algorithm \ref{alg:csimple} and consists of two stages. In the first stage, each party independently obtains its approximate Pareto solution set, denoted as \(\Phi_1^{\varepsilon_1}\) and \(\Phi_2^{\varepsilon_2}\). In the second stage, under a relaxed approximation policy, one party, assumed here to be party 1, proposes its approximate Pareto solutions \(\mathbf{x} \in \Phi_1^{\varepsilon_1}\) sequentially. Party 2 then decides whether to relax its approximation ratio \(\varepsilon_2\) and accept the proposed solution as a common solution. This interaction is modeled as an ultimatum game \cite{nowak2000fairness}.

\begin{algorithm}[!h]
	\caption{ Simple Evolutionary Multi-party Multi-objective Optimizer ($\text{EMPMO}_{\text{simple}}^{SP}$) } 
	\label{alg:csimple}
	\begin{algorithmic}[1]
		\For {$m \gets 1$ \textbf{to} $M$}
		\State Set initial population $P_m \leftarrow \{ (1)\}$ for each party;
		\EndFor
		\While {termination condition is not met}
		\For {$m \gets 1$ \textbf{to} $M$}
		\State Randomly select an individual $\mathbf{x} \in P_m$;
		\State Mutate $\mathbf{x}$ to get $\mathbf{x}'$;
		\If{$\nexists \mathbf{z} \in P_m$ such that $F_m(\mathbf{z}) \succ F_{m}(\mathbf{x}^{\prime}) \vee  b_r(F_{m}(\mathbf{z})) \succ b_r(F_{m}(\mathbf{x}^{\prime}))$}
		\State $P_m\leftarrow(P_m \setminus \{\mathbf{z} \in P_m \mid b_r(F_{m}(\mathbf{x}^{\prime}))\succeq
		b_r(F_{m}(\mathbf{z}))\})\cup\{\mathbf{x}'\}$;
		\EndIf
		\EndFor
		\EndWhile
		\For {$i \gets 1$ \textbf{to} $\mid P_1 \mid$} 
		\State Calculate the minimum $\varepsilon_{2,i}$ corresponding to $P_1(i,:)$
		\If {$ \varepsilon_{2,i} \leq \varepsilon_{2}^{max} $}
		\State $ \varepsilon_{2,i}^{\prime} = \varepsilon_{2,i} $
		\EndIf
		\EndFor
		\State Find the minimum $\varepsilon_{2,i}^{\prime}$ and its corresponding solution for each endpoint in $P_1$ 
		\State All solutions found constitute the final solution set 
	\end{algorithmic}
\end{algorithm}

\begin{definition}[Ultimatum Game]\label{def:uGame}
	For BPMOSP, the process of obtaining an approximate common Pareto set from the two parties' approximate Pareto sets, \(\Phi_1^{\varepsilon_1}\) and \(\Phi_2^{\varepsilon_2}\), is modeled as an ultimatum game \(\{\mathcal{M}, \Omega, \mathcal{U}\}\), where:
	\begin{itemize}
		\item \(\mathcal{M} = \{1, 2\}\) denotes the set of decision-makers (parties), with party 1 acting as the proposer and party 2 as the responder.
		\item \(\Omega = \{\Omega_1, \Omega_2\} \), where \( \Omega_1 = \text{Paths}(s, v)\) is the policy space of party 1, and \( \Omega_2 = \{\text{Accept}, \text{Reject}\}\) is the policy space of party 2.
		\item \(\mathcal{U} = (u_1, u_2)\) represents the utility functions of the two parties. For party 1, the utility function is:
		\[
		u_1(p, y) =
		\begin{cases}
			U_1^{\max}, & \text{if } p \in \Phi_1^{\varepsilon_1}\, \text{and}\, y=\text{Accept}, \\
			0, & \text{otherwise},
		\end{cases}
		\]
		where \(U_1^{\max}\) is the maximum utility when \(p \in \Phi_1^{\varepsilon_1}\).
		
		For party 2, the utility function is:
		\[
		u_2(p,y) =
		\begin{cases}
			U_2^{\max}, & \text{if } p \in \Phi_2^{\varepsilon_2}\ \\ & \text{and}\ =\text{Accept}, \\
			U_2^{\max} \frac{\varepsilon_2^{\max} - \varepsilon_2^{\prime}}{\varepsilon_2^{\max} - \varepsilon_2}, & \text{if } p \in \Phi_2^{\varepsilon_2^{\max}}, p \notin \Phi_2^{\varepsilon_2}\ \\ & \text{and}\  y=\text{Accept}, \\
			0, & \text{otherwise},
		\end{cases}
		\]
		where \(U_2^{\max}\) is the maximum utility when \(p \in \Phi_2^{\varepsilon_2}\), \(\varepsilon_2^{\max}\) is the maximum approximation ratio threshold acceptable to party 2, and $\varepsilon_2^{\prime}\in[\varepsilon_2, \varepsilon_2^{\max}]$ indicates the ratio to which path \(\mathbf{x}\) approximates a Pareto-optimal solution for party 2. The linear term ensures that the utility decreases proportionally as \(p\) deviates from \(\varepsilon_2\) within the range \([\varepsilon_2, \varepsilon_2^{\max}]\).
	\end{itemize}
\end{definition}

Definition 11 is essentially a variation of the ultimatum game. When two parties cannot reach a consensus, one of them will relax its conditions to seek cooperation, but will not exceed the bottom line. In this definition, we assume that party 2 will relax his conditions, and his benefits will also decrease when he relaxes his conditions. According to the payoff functions of both parties, when consensus is reached, both parties achieve the highest payoff, and the corresponding strategy is the Nash equilibrium.

\begin{theorem}
	Let \( \varepsilon_1 > 0 \), \( \varepsilon_2 > 0 \), \( \varepsilon = \min\{\varepsilon_1, \varepsilon_2\} \), and $1 + \varepsilon_1$ and $1 + \varepsilon_2$ represent the approximation ratios for party 1 and 2, respectively. Let $\delta=\frac{n \log \left(n w^{\max}\right)}{\log(1+\varepsilon)}$ and \( k = \max\{k_1, k_2\} \) represent the maximum number of objectives among all parties. The expected number of generations of  \( \text{EMPMO}^{\text{SP}}_{\text{simple}} \)  applied to BPMOSP, which achieves the Nash equilibrium of the ultimatum game for both parties is bounded by:
	$$
	O\left(n^3 \cdot \delta^{k-1} \log\left(n \delta^{k-1}\right)\right).
	$$
\end{theorem}
\begin{proof}
	In the first phase of \( \text{EMPMO}^{\text{SP}}_{\text{simple}} \), the Pareto approximate solution sets for both parties are independently searched. In Ref. \cite{horoba2010exploring}, the entire process is divided into $n-1$ stages based on path length. In each stage, a path is selected for mutation until the population reaches the conditions for approximate dominance at that stage. The mutation strategy used is to add or delete nodes with equal probability only at the endpoints of the path.
	We removed two of the simplification operations and obtained the running time for each party is bounded by:
	\[
	O\left(n^3 \cdot \delta_m^{k_m-1} \cdot \log\left(n \delta_m^{k_m-1}\right)\right),
	\]
	where $\delta_m=\frac{n \log \left(n w_m^{\max}\right)}{\log(1+\varepsilon_m)}$. Thus, the overall running time for the first phase is the sum of the running times for both parties:
	\[
	\begin{aligned}
		&O\left(\sum_{m=1}^2 n^3 \cdot \delta_m^{k_m-1} \log\left(n \delta_m^{k_m-1}\right)\right) \\
		\leq & O\left(n^3 \cdot \delta^{k-1} \log\left(n \delta^{k-1}\right)\right),
	\end{aligned}
	\]
	where $\delta=\frac{n \log \left(n w^{\max}\right)}{\log(1+\varepsilon)}$. In the second phase, we sequentially check each of the first party's approximate Pareto solutions to verify whether it satisfies the Nash equilibrium of the ultimatum game as defined in Definition 11. This step does not require additional evaluation rounds.
\end{proof}

\subsection{Runtime Analysis of Consensus-based Evolutionary Bi-party Multi-objective Optimizer}
The \( \text{EMPMO}^{\text{SP}}_{\text{simple}} \) algorithm can be extended to solve the BPMOSP with asymmetric influence between decision-makers, where one party is required to relax its predefined approximation ratio. To promote fairness, the evolutionary process aims to maximize consensus between the two parties, thereby facilitating a more balanced negotiation.

It is well established that when a path's weight is determined by a single value, the shortest path possesses the optimal substructure property \cite{bellman1958routing}, meaning that every subpath of the shortest path is itself a shortest path. Similarly, in the case of the common shortest path with multiple weights for two parties, analogous properties hold. Lemma \ref{optimal substructure} formalizes the optimal substructure property for the common shortest path.

\begin{lemma} \label{optimal substructure} 
	Given a directed weighted graph \( G = (V, E, W) \), if path $\mathbf{x}=(v_{0},v_{1},...,v_{l-1},v_{l})$ is a common shortest path from source vertex $v_{0}$ to target vertex $v_{l}$, then $\mathbf{x}'=(v_{0},v_{1},...,v_{l-1})$ is a common shortest path from source vertex $v_{0}$ to target vertex $v_{l-1}$.
\end{lemma}
\begin{proof}
	Let \( \mathbf{x} \) be a common shortest path from source vertex \( v_0 \) to target vertex \( v_l \). This implies that, among all other paths from \( v_0 \) to \( v_l \), no path dominates \( \mathbf{x} \) for all parties simultaneously. Let \( F_m(\mathbf{x}) = (f_{m1}, \dots, f_{mk_m}) \) and \( w_m(v_{l-1}, v_l) = (w_{m1}, \dots, w_{mk_m}) \). Then, we have \( F_m(\mathbf{x}') = (f_{m1} - w_{m1}, \dots, f_{mk_m} - w_{mk_m}) \).
	
	Assume, for contradiction, that \( \mathbf{x}' \) is not a common shortest path from \( v_0 \) to \( v_{l-1} \). This means there exists a path \( \mathbf{z}' = (v_0, v_1, \dots, v_{l-1}) \) from \( v_0 \) to \( v_{l-1} \) that dominates \( \mathbf{x}' \) for all parties. Let \( F_m(\mathbf{z}') = (f_{m1}^\prime, \dots, f_{mk_m}^\prime) \). By Definition 2, we have \( f_{mk}^\prime \leq f_{mk} - w_{mk} \) for all \( 1 \leq k \leq k_m \), and \( F_m(\mathbf{z}') \neq F_m(\mathbf{x}') \).
	
	Using \( \mathbf{z}' \), we can construct a new path \( \mathbf{z} = (v_0, v_1, \dots, v_{l-1}, v_l) \) with \( F_m(\mathbf{z}) = (f_{m1}^\prime + w_{m1}, \dots, f_{mk_m}^\prime + w_{mk_m}) \). Since \( f_{mk}^\prime \leq f_{mk} - w_{mk} \) for all \( 1 \leq k \leq k_m \), it follows that \( F_m(\mathbf{z}) \succeq F_m(\mathbf{x}) \), which contradicts the assumption that \( \mathbf{x} \) is a common shortest path. Therefore, the assumption is false, and \( \mathbf{x}' \) must be a common shortest path from \( v_0 \) to \( v_{l-1} \).
\end{proof}

Lemma \ref{optimal substructure} establishes that a common solution exists at each path length. Building on this result, if a common solution is found at every path length, then as the path length increases, a common solution to each target vertex can be identified. Therefore, we propose a consensus-based evolutionary multi-party multi-objective optimizer (\(\text{EMPMO}^{\text{SP}}_{\text{cons}}\)), as presented in Algorithm \ref{alg:ctoc} for the BPMOSP. The details are as follows:

\begin{algorithm}[!h]
	\caption{ Consensus-based Evolutionary Multi-party Multi-objective Optimizer ($\text{EMPMO}_{cons}^{SP}$) } 
	\label{alg:ctoc}
	\begin{algorithmic}[1]
		\Require Parameter $r$ controlling the box size
		\State Initialize the population $P\leftarrow \{ (1)\}$;
		\While {termination condition is not met}
		\State Randomly select an individual $\mathbf{x} \in P$;
		\State Apply the mutation operator on $\mathbf{x}$ to get $\mathbf{x}'$;
		\If{$\exists m, \nexists \mathbf{z} \in P, F_m(\mathbf{z}) \succ F_{m}(\mathbf{x}^{\prime}) \vee  b_r(F_{m}(\mathbf{z})) \succ b_r(F_{m}(\mathbf{x}^{\prime}))$}
		\State $P\leftarrow(P \setminus \{\mathbf{z} \in P \mid \forall m',b_r(F_{m'}(\mathbf{x}^{\prime}))\succeq
		b_r(F_{m'}(\mathbf{z}))\})\cup\{\mathbf{x}'\}$; 
		\EndIf
		\EndWhile
	\end{algorithmic}
\end{algorithm}

During the initialization of the population, we set it to \( \{(1)\} \). Drawing inspiration from common one-bit mutation, we define a novel mutation operator tailored for the BPMOSP, as detailed in Definition \ref{ref:mutation operator}. This mutation operator randomly selects a vertex within the path to mutate. In each iteration, a solution is randomly chosen for mutation. Subsequently, the population is updated. If no solution in the population dominates the mutated solution, and no box index corresponding to any party dominates that of the mutated solution, the new solution is added to the population. Subsequently, any solutions that are dominated by the new solution or share the same box index for all parties are removed from the population. The input parameter r in the algorithm represents the size of the box, and its value is set to $(1+\varepsilon)^{1/(n-1)}$ in subsequent runtime analysis and experiments.

\begin{definition} \label{ref:mutation operator}
	Each mutation operation randomly and uniformly performs one of the following two operations:
	
	\textbf{Add}: Randomly select a vertex \( v_i \) in the path \( \mathbf{x} = (v_0, v_1, \ldots, v_{l-1}, v_l) \), where \( 0 \leq i \leq l \).
	\begin{itemize}
		\item If \( v_i \neq v_l \) (i.e., \( v_i \) is not the target vertex of the path), and there exists \( v' \in V \) such that \( (v_i, v') \in E \) and \( (v', v_{i+1}) \in E \), the new path after mutation is \( \mathbf{x}' = (v_0, \ldots, v_i, v', v_{i+1}, \ldots, v_l) \).
		\item If \( v_i = v_l \) (i.e., \( v_i \) is the target vertex of the path), and there exists \( v' \in V \) such that \( (v_l, v') \in E \), the new path after mutation is \( \mathbf{x}' = (v_0, v_1, \ldots, v_l, v') \).
	\end{itemize}
	
	\textbf{Delete}: Randomly select a vertex \( v_i \) in the path \( \mathbf{x} = (v_0, v_1, \ldots, v_{l-1}, v_l) \), where \( 1 \leq i \leq l-1 \).
	\begin{itemize}
		\item If \( 1 \leq i \leq l-2 \) and \( (v_i, v_{i+2}) \in E \), the new path after mutation is \( \mathbf{x}' = (v_0, \ldots, v_i, v_{i+2}, \ldots, v_l) \).
		\item If \( i = l-1 \), the new path after mutation is \( p' = (v_0, \ldots, v_{l-1}) \).
	\end{itemize}
\end{definition}

A more detailed explanation is provided in Example \ref{example2}.

\begin{example} \label{example2}
	Using the weighted directed graph in Figure \ref{fig:np}, the common solutions from vertex \(1\) to all other vertices can be determined as follows:
	\begin{itemize}
		\item From vertex \(1\) to vertex \(2\): \( (1, 2) \),
		\item From vertex \(1\) to vertex \(3\): \( (1, 3) \),
		\item From vertex \(1\) to vertex \(4\): \( (1, 3, 4) \),
		\item From vertex \(1\) to vertex \(5\): \( (1, 3, 4, 5) \).
	\end{itemize}
	
	Let \( \mathbf{x} = (1, 3, 4) \) and \( \mathbf{z} = (1, 4) \).
	
	\begin{table}[h]
		\centering
		\caption{All possible paths from source vertex \(1\) to target vertices except \(5\).}
		\label{tab: G1}
		\begin{tabular}{ccc}
			\toprule
			Paths          & Objective Value for Party \(1\) & Objective Value for Party \(2\) \\ \midrule
			\( \mathbf{(1, 2)} \)  & \( \mathbf{(1, 2)} \) & \( \mathbf{(2, 4)} \) \\ \midrule
			\( (1, 2, 3) \)  & \( (4, 5) \)           & \( (4, 5) \) \\
			\( \mathbf{(1, 3)} \)  & \( \mathbf{(3, 2)} \) & \( \mathbf{(3, 5)} \) \\ \midrule
			\( \mathbf{(1, 3, 4)} \) & \( \mathbf{(5, 3)} \) & \( \mathbf{(4, 6)} \) \\
			\( (1, 2, 3, 4) \) & \( (6, 6) \)          & \( (5, 6) \) \\
			\bottomrule
		\end{tabular}
	\end{table}
	
	During the mutation process, the path length increases incrementally. When the length is \(1\), non-common solutions like \(\mathbf{z}\) are generated, but the common solution \( \mathbf{x}' = (1, 3) \) is present in the population. Although a common solution from \(1\) to \(4\) is not immediately available, the population contains the common solution from \(1\) to \(3\). Through mutation, path \(\mathbf{x}\) (a common solution from \(1\) to \(4\)) can be identified.
	
	According to the update rules defined in Definition \ref{ref:mutation operator}, path \(\mathbf{z}\) will naturally be removed from the population. Even if \(\mathbf{x}\) is not obtained through mutation, the population will retain at least one solution satisfying conditions for each endpoint. If \(\mathbf{z}\) is retained, it indicates that it satisfies the approximate domination criteria for both parties, allowing a solution that meets the conditions to still be found.
	
	Thus, for any path length and target vertex, it suffices to identify at least one new common solution at each step.
\end{example}

To analyze the expected running time of $\text{EMPMO}^{\text{SP}}_{\text{cons}}$ on BPMOSP, several lemmas are required to assist in the proof.

\begin{lemma} \label{pmax}
	Let $r>1$ and $m \in \{1, 2\}$. When using $\text{EMPMO}^{\text{SP}}_{\text{cons}}$ to optimize all $F_m$, the maximum population size is upper bounded by
	$$\min_{1\le m \le M}(n-1)\cdot(\lfloor\log_{r}((n-1)w_{m}^{\max})\rfloor+1)^{k_{m}-1}+1$$
\end{lemma}
\begin{proof}
	To determine the maximum population size, we divide the analysis into two parts. The first part considers the optimization of any single \( F_m \), while the second part extends the analysis to all \( F_m \).
	
	Let \( \mathrm{Path}_{1,i} \) represent all possible paths from source vertex \( 1 \) to target vertex \( i \). For an objective value \( f_{mk} \), there is a cardinality \( |b_r(f_{mk}(\mathrm{Path}_{1,i}))| \) for the corresponding \( m \)-th box. Given the following:
	$$\begin{aligned}
		|b_{r}(f_{mk}(\mathrm{Path}_{1,i}))| &\leq b_r(f_{mk}^{\max}) - b_r(f_{mk}^{\min}) + 1\\
		&= b_r((n-1)w_{m}^{\max})-b_r(1)+1 \\
		& =\lfloor\log_{r}\left(\left(n-1\right)w_{m}^{\max}\right)\rfloor+1.
	\end{aligned}$$
	Considering the population selecting rules in $\text{EMPMO}^{\text{SP}}_{\text{cons}}$, at most one candidate solution from each non-dominated box can enter the population. Therefore, we have:
	$$\begin{aligned}
		|b_r(F_{m}(\mathrm{Path}_{1,i}))| & \leq \prod_{k=2}^{k_m}|b_r(f_{mk}(\mathrm{Path}_{1,i}))| \\
		& \leq (\lfloor \log_{r}((n-1)w_{m}^{\max})\rfloor+1)^{k_{m}-1}
	\end{aligned}$$
	
	Thus, for the \( m \)-th party, \( \mathrm{Path}_{1,i} \) can have at most \( \left( \lfloor \log_r \left( (n-1) w_m^{\max} \right) \rfloor + 1 \right)^{k_m - 1} \) distinct boxes that can enter the population. Excluding the case where the source vertex and the target vertex are the same, there are \( n - 1 \) distinct target vertices. Therefore, for any \( F_m \), there are at most \( (n-1) \cdot \left( \lfloor \log_r \left( (n-1) w_m^{\max} \right) \rfloor + 1 \right)^{k_m - 1} \) boxes that can be included in the population.
	
	In the second part, the condition for saving a solution is stricter than considering just one party alone. Hence, when optimizing all \( F_m \), the maximum population size is:
	\[
	\min_{1\le m \le M} (n-1) \cdot \left( \lfloor \log_r \left( (n-1) w_m^{\max} \right) \rfloor + 1 \right)^{k_m - 1} + 1
	\]
	The final constant \( 1 \) accounts for the case where the path length is zero, corresponding to the set \( \{(1)\} \).
\end{proof}

\begin{lemma} \label{mutate}
	Let \( r > 1 \), \( m \in \{1, 2\} \), and \( \mathbf{x} = (v_{0}, v_{1}, \dots, v_{i-1}, v_{i}) \) be a path of length \( l \), where \( 0 \leq i \leq n - 2 \). Let \( \mathbf{z} = (v_{0}, u_{1}, \dots, u_{j-1}, u_{j}) \in P \) be a path in the population \( P \) such that \( u_j = v_i \). If \( F_{m}(\mathbf{z}) \succeq_{r^i} F_{m}(\mathbf{x}) \) for all \( m \), then \( F_{m}(\mathbf{z}') \succeq_{r^i} F_{m}(\mathbf{x}') \) for all \( m \), where \( \mathbf{x}' = (v_{0}, v_{1}, \dots, v_{i}, v_{i+1}) \) and \( \mathbf{z}' = (v_{0}, u_{1}, \dots, u_{j}, v_{i+1}) \).
\end{lemma}
\begin{proof}
	Given that \( F_{m}(\mathbf{z}) \succeq_{r^i} F_{m}(\mathbf{x}) \), it follows that \( f_{mk}(\mathbf{z}) \leq r^{i} \cdot f_{mk}(\mathbf{x}) \) for all \( 1 \leq k \leq k_m \). For the extended paths \( \mathbf{z}' \) and \( \mathbf{x}' \), we have:
	\[
	f_{mk}(\mathbf{z}') = f_{mk}(\mathbf{z}) + w_{mk}((u_{j}, v_{i+1}))
	\]
	\[
	f_{mk}(\mathbf{x}') = f_{mk}(\mathbf{x}) + w_{mk}((v_{i}, v_{i+1})).
	\]
	
	From \( f_{mk}(\mathbf{z}) \leq r^{i} \cdot f_{mk}(\mathbf{x}) \), it follows that:
	\[
	\begin{aligned}
		f_{mk}(\mathbf{z}) + w_{mk}(v_{i}, v_{i+1}) & \leq r^{i} \cdot f_{mk}(\mathbf{x}) + w_{mk}(v_{i}, v_{i+1}) \\
		& \leq r^{i} \cdot f_{mk}(\mathbf{x}) + r^{i} \cdot w_{mk}(v_{i}, v_{i+1}).
	\end{aligned}
	\]
	
	Thus:
	\[
	f_{mk}(\mathbf{z}') \leq r^{i} \cdot f_{mk}(\mathbf{x}'),
	\]
	
	which implies that \( F_{m}(\mathbf{z}') \succeq_{r^i} F_{m}(\mathbf{x}') \).
\end{proof}

\begin{lemma} \label{grow}
	Let \( r > 1 \), \( m \in \{1, 2\} \), and \( \mathbf{x} = (v_1, \ldots, v_i) \) be a path in the solution space. The set \( P_i \) consists of paths in the population \( P \) whose target vertices are \( v_i \). If \( P_i \) satisfies \( F_{m}(P_i) \succeq_{r^j} F_{m}(\mathbf{x}) \) for all \( m \), then the set \( P_i' \), consisting of paths with the same target vertex as \( \mathbf{x} \) in the subsequent population \( P' \) obtained through mutation, satisfies \( F_{m}(P_i') \succeq_{r^{j+1}} F_{m}(\mathbf{x}) \) for all \( m \).
\end{lemma}
\begin{proof}
	From the update rule of the population in Algorithm 6, for any \( \mathbf{x} \in P_i \subset P \), there exists \( \mathbf{z}' \in P_i' \subset P' \) such that \( b_r(F_{m}(\mathbf{z}')) \succeq b_r(F_{m}(\mathbf{z})) \) for all \( m \).
	
	Given \( b_r(F_{m}(\mathbf{z}')) \succeq b_r(F_{m}(\mathbf{z})) \) and \( F_{m}(\mathbf{z}) \succeq_{r^j} F_{m}(\mathbf{x}) \) (since \( F_{m}(P_i) \succeq_{r^j} F_{m}(\mathbf{x}) \) and \( \mathbf{z} \in P_i \)), it follows that:
	\[
	\lfloor \log_{r}(f_{mk}(\mathbf{z}')) \rfloor \leq \lfloor \log_{r}(f_{mk}(\mathbf{z})) \rfloor,
	\]
	then we have
	\[
	\frac{\log(f_{mk}(\mathbf{z}'))}{\log r} - 1 \leq \frac{\log(f_{mk}(\mathbf{z}))}{\log r},
	\]
	and then
	\[
	\log(f_{mk}(\mathbf{z}')) \leq \log(r \cdot f_{mk}(\mathbf{z})).
	\]
	Thus,
	\[
	f_{mk}(\mathbf{z}') \leq r \cdot f_{mk}(\mathbf{z}).
	\]
	
	By induction, since \( f_{mk}(\mathbf{z}) \leq r^j f_{mk}(\mathbf{x}) \), it follows that:
	\[
	f_{mk}(\mathbf{z}') \leq r \cdot r^j f_{mk}(\mathbf{x}) = r^{j+1} f_{mk}(\mathbf{x}).
	\]
	This demonstrates that \( F_{m}(P_i') \succeq_{r^{j+1}} F_{m}(\mathbf{x}) \) for all \( m \).
\end{proof}

\begin{theorem}
	Let \( \varepsilon > 0 \), \( m \in \{1, 2\} \), $\delta=\frac{n \log \left(n w^{\max}\right)}{\log(1+\varepsilon)}$ and \( k \) represent the maximum number of objectives among all parties. When \( \text{EMPMO}_{cons}^{SP} \) is applied to BPMOSP, it achieves a \((1 + \varepsilon)\)-approximation with an expected number of generations bounded by:
	\[
	 O\left(n^4 \cdot \delta^{k-1} \cdot \log\left(n\delta^{k-1}\right) \right).
	\]
\end{theorem}

\begin{proof}
	Since there are \( n \) vertices in the graph, the maximum path length is \( n-1 \). Based on the path length, we divide the entire evolution process into \( n-1 \) stages. At the end of the \( i \)-th stage, we have \( F_{m}(P) \succeq_{r^{i+1}} F_{m}(S_{i+1}) \) for all \( m \), where \( P \) represents the population, \( 0 \leq i \leq n-2 \) is the path length, \( S_{i+1} \) represents all possible paths in the solution space with path lengths not exceeding \( i+1 \), and \( r = (1+\varepsilon)^{\frac{1}{n-1}} \). After completing all \( n-1 \) stages, it holds that \( r^{n-1} = 1+\varepsilon \).
	
	At the initialization of the population, let \( P = \{(v_1)\} = S_{0} \), which satisfies \( F_{m}(P) \succeq_{r^{0}} F_{m}(S_{0}) \) for all \( m \). At the beginning of the \( i \)-th stage, it holds that \( F_{m}(P) \succeq_{r^{i}} F_{m}(S_{i}) \) for all \( m \).
	
	For any \( \mathbf{x} \in S_i \), according to Lemma \ref{mutate}, there exists a path \( \mathbf{z} \) in the population with the same target vertex as \( \mathbf{x} \) such that \( F_{m}(\mathbf{z}') \succeq_{r^{i}} F_{m}(\mathbf{x}') \) for all \( m \), where \( \mathbf{x}' \in S_{i+1} \) and \( \mathbf{z}' \) is a mutation of \( \mathbf{z} \). By Lemma \ref{grow}, the subsequent population \( P' \) satisfies \( F_{m}(P') \succeq_{r^{i+1}} F_{m}(\mathbf{x}') \) for all \( m \). The \( i \)-th stage ends when \( F_{m}(P') \succeq_{r^{i+1}} F_{m}(S_{i+1}) \) is satisfied for all \( m \).
	
	According to Lemma \ref{optimal substructure}, we keep the form of $F_m(\mathbf{x})$, $w_m(v_{l-1}, v_l)$ and $F_m(\mathbf{x}')$ unchanged and let  \( F_m(\mathbf{z}) = (f_{m1}^\prime, \dots, f_{mk_m}^\prime) \) and \( F_m(\mathbf{z}') = (f_{m1}^\prime - w_{m1}, \dots, f_{mk_m}^\prime - w_{mk_m} )  \).
	If \( \mathbf{x}' \) is a approximate common shortest path from \( v_0 \) to \( v_{l-1} \) but \( \mathbf{z}' \) is not, we have \( F_m(\mathbf{x}') \succeq_{r^{j}} F_m(\mathbf{z}') \) ,that is \( (f_{mk} - w_{mk} ) \leq r^j(f_{mk}^{\prime} - w_{mk} )\) for all \( 1 \leq k \leq k_m \). After scaling the inequality, we can get $
	f_{mk} \leq r^jf_{mk}^{\prime} - (r^j - 1)w_{mk} \leq r^{j+1}f_{mk}^{\prime} - (r^j - 1)w_{mk} \leq r^{j+1}f_{mk}^{\prime}$
	because of $ r > 1$, that is \( F_m(\mathbf{x}) \succeq_{r^{j+1}} F_m(\mathbf{z}) \). Therefore, we know that a common solution can be found in the first stage, and an approximate common solution of the corresponding length can be found through mutation in the \( i \)-th stage. If this solution is deleted during population evolution, it means that there are other approximate common solutions that are closer to the true common solution than it.
	
	For the mutation operation, the probability of selecting a solution from the population is \( \frac{1}{|P|} \), the probability of performing the Add operation in the mutation operator on this solution is \( \frac{1}{2} \), and the probability of selecting the endpoint of the path and adding an edge is at least \( \frac{1}{n} \cdot \frac{1}{n-1} \). Thus, the probability of successfully executing the operation in Lemma \ref{mutate} is at least \( \frac{1}{2 \cdot |P| \cdot n(n-1)} \geq \frac{1}{2n(n-1)P_{\max}} \). According to Lemma \ref{pmax},
	\(
	P_{\max} = \min_{1\le m\le M} (n-1) \cdot \left(\lfloor \log_{r}((n-1)w_{m}^{\max}) \rfloor + 1\right)^{k_{m}-1} + 1.
	\) 
	
	Hence, the maximum expected number of generations for the \( i \)-th stage is
	\[
	2n(n-1) \cdot P_{\max} \cdot \sum_{j=1}^{P_{\max}} \frac{1}{j} = O\left(n^{2} \cdot P_{\max} \cdot \log P_{\max}\right).
	\]
	
	Let \( k = \max\{k_{1}, \ldots, k_{M}\} \). Since \( r = (1+\varepsilon)^{\frac{1}{n-1}} \), \( w_m^{\max} \leq w^{\max} \), we have
	\[
	\begin{aligned}
		&O\left(n^{2} \cdot P_{\max} \cdot \log P_{\max}\right) \\
		=&O\left(n^{2} \cdot n \left(\log_{r}(n w^{\max})\right)^{k-1} \log\left(n \cdot \left(\log_{r}(n w^{\max})\right)^{k-1}\right)\right) \\
		=&O\left(n^{3}\left(\frac{n \log(n w^{\max})}{\log(1+\varepsilon)}\right)^{k-1} \log\left(n \left(\frac{n \log(n w^{\max})}{\log(1+\varepsilon)}\right)^{k-1}\right)\right) \\
		=&O\left(n^3 \cdot \delta^{k-1} \cdot \log\left(n\delta^{k-1}\right) \right).
	\end{aligned}
	\]  
	
%

	Since there are \( n-1 \) stages in total, the overall expected number of generations is at most
	\[O\left(n^4 \cdot \delta^{k-1} \cdot \log\left(n\delta^{k-1}\right) \right).
	\]  
\end{proof}

\subsection{Comparison Between Bi-party Multi-objective Shortest Path and Multi-objective Shortest Path}
Consider the case where \( k_1 = k_2 = 2 \), representing a bi-party bi-objective shortest path (SP) problem. Compared to the generalized four-objective SP problem, the time complexities are as follows.

Bi-objective SP Problems: As stated in \cite{horoba2010exploring}, the upper bound on the expected number of generations is:
\[
O\left(n^3 \cdot \frac{n \log(n w^{\max})}{\log(1+\varepsilon) } \cdot \log\left(n \cdot \frac{n \log(n w^{\max})}{\log(1+\varepsilon)}\right)\right).
\]

This bound also applies to \(\text{EMPMO}_{\text{simple}}^{\text{SP}}\) for the bi-party bi-objective SP problem, but with a larger approximation ratio.

Four-objective Problems: The upper bound on  the expected number of generations is:
\[
O\left(n^3 \cdot \left(\frac{n \log(n w^{\max})}{\log(1+\varepsilon)}\right)^3 \cdot \log\left(n \cdot \left(\frac{n \log(n w^{\max})}{\log(1+\varepsilon)}\right)^{3} \right)\right).
\]

Bi-party Bi-objective SP Problem: Using  \(\text{EMPMO}_{\text{cons}}^{\text{SP}}\) , the upper bound on the expected number of generations is:
\[
O\left(n^4 \cdot \frac{n \log(n w^{\max})}{\log(1+\varepsilon)} \cdot \log\left(n \cdot \frac{n \log(n w^{\max})}{\log(1+\varepsilon)}\right)\right).
\]

In all cases, the approximation accuracy \(\varepsilon\) is consistent across the bi-objective, four-objective, and bi-party bi-objective problems. While  \(\text{EMPMO}_{\text{cons}}^{\text{SP}}\) modifies the mutation operator, its running time remains lower than that for the four-objective problem and slightly higher than that for the bi-objective problem. Importantly, unlike  \(\text{EMPMO}_{\text{simple}}^{\text{SP}}\), \(\text{EMPMO}_{\text{cons}}^{\text{SP}}\)  does not relax the approximation accuracy, thereby yielding higher-quality solutions.

\section{Experiments}
\label{sec:experiments}

In this section, we conduct short experiments on the artificial problem BPAOAZ and the bi-party multi-objective UAV path planning (BPUAVPP) to complement the theoretical results.

\subsection{On Artificial Problem BPAOAZ}

In this subsection, we analyze the average running times of the MOEA, i.e., SEMO, and three MPMOEAs: \( \text{EMPMO}_{\text{simple}} \), \( \text{EMPMO}_{\text{random}} \), and \( \text{EMPMO}_{\text{payoff}} \), on the artificial problem \( \mathrm{BPAOAZ} \). Each algorithm was independently executed 10 times. The results, illustrated in Fig. \ref{fig:runtime_p}, provide a comparative evaluation of their performance.

SEMO is designed to identify the complete Pareto set of the associated multi-objective optimization problem, whereas the other three algorithms focus on finding the common Pareto set. Among these, \( \text{EMPMO}_{\text{payoff}} \) serves as a theoretical performance baseline for comparison. The results demonstrate that SEMO exhibits the longest runtime, followed by \( \text{EMPMO}_{\text{simple}} \), \( \text{EMPMO}_{\text{random}} \), and \( \text{EMPMO}_{\text{payoff}} \), in descending order of execution time. These experimental findings corroborate the theoretical results presented in Section~\ref{sec: theoretical analysis}.

It is noteworthy that Theorem \ref{the:random} establishes that the expected runtime of  \( \text{EMPMO}_{\text{random}} \) on \( \mathrm{BPAOAZ} \) is bounded by \( O\left(\left(\frac{1}{\varphi} + \frac{1}{1-\varphi} \right) \frac{n^2}{2} \log n \right) \), consistent with the runtime complexity of \( \text{EMPMO}_{\text{simple}} \). However, as shown in Fig. \ref{fig:runtime_p}, when the parameter \( \varphi = 0.5 \) is set, the actual runtime of  \( \text{EMPMO}_{\text{random}} \) is closer to the theoretical lower bound \( \Omega(n \log n) \) and is comparable to that of \( \text{EMPMO}_{\text{payoff}} \) .

This discrepancy arises from the significant influence of the selection probability \( \varphi \) on runtime. Fig. \ref{fig:runtime_phi} further illustrates the runtime of \( \text{EMPMO}_{\text{random}} \)  as a function of \( \varphi \) across different problem sizes. For \( \mathrm{BPAOAZ} \), which features symmetric subproblems for both parties, \( \varphi = 0.5 \) is empirically identified as the optimal parameter. Additionally, as the problem size increases, the variability in runtime becomes more pronounced, peaking at the extremes of the parameter range. Due to the added randomness in the selection process,  \( \text{EMPMO}_{\text{random}} \) exhibits greater variance compared to  \( \text{EMPMO}_{\text{simple}} \). This is evident from the broader shaded regions in Fig. \ref{fig:runtime_phi}, which represent the range of runtime fluctuations. In addition, for a more rigorous proof, we add two experiments. One experiment (Fig. \ref{fig:p_checkphi}) shows the comparison results when $\varphi=0.95$ and further proves that the value of $\varphi$ will affect the running time of \( \text{EMPMO}_{\text{random}} \) to find the common solution. The other experiment (Fig. \ref{fig:p_checkrandom}) simulates the derivation process of \( \text{EMPMO}_{\text{random}} \) and proves the correctness of the lower bound of the running time.

\begin{figure}
	\centering
	\small
	\subfloat[Logarithmic variation of running time with problem size]{\includegraphics[width=0.4\columnwidth]{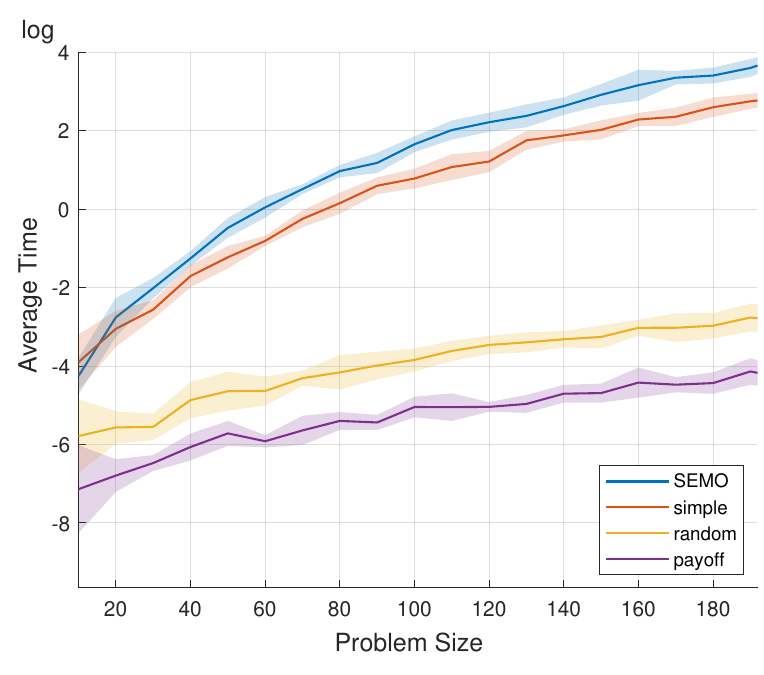}
		\label{fig:runtime_p}}
	\subfloat[Variation of \( \text{EMPMO}_{\text{random}} \) running time with \( \varphi \)]{\includegraphics[width=0.43\columnwidth]{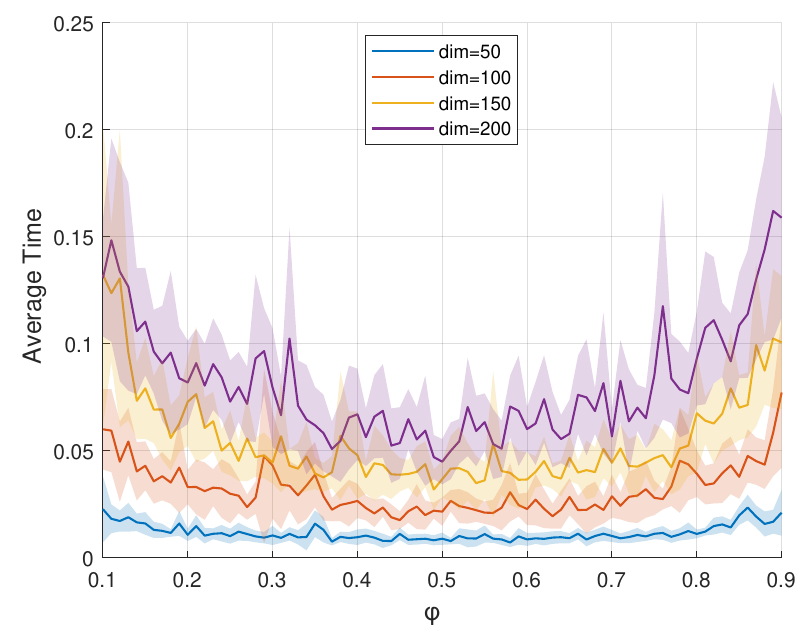}
		\label{fig:runtime_phi}}
	\caption{The average runtime of SEMO, \( \text{EMPMO}_{\text{simple}} \), \( \text{EMPMO}_{\text{random}} \), and \( \text{EMPMO}_{\text{payoff}} \) on artificial problem \( \mathrm{BPAOAZ} \) and the y-axis is the average runtime in base 10.}
	\label{fig:GRandRDR}
\end{figure}

\begin{figure}[h]
	\centering
	\small
	\subfloat[The running time of \( \text{EMPMO}_{\text{random}} \) with $\varphi=0.95$]{\includegraphics[width=0.46\columnwidth]{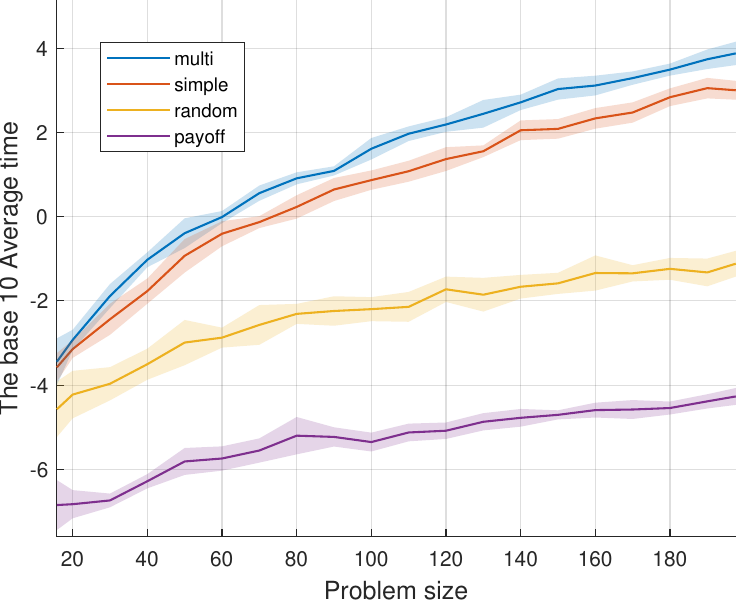}
		\label{fig:p_checkphi}}
	\subfloat[Variation of the running time of \( \text{EMPMO}_{\text{random}} \) ]{\includegraphics[width=0.43\columnwidth]{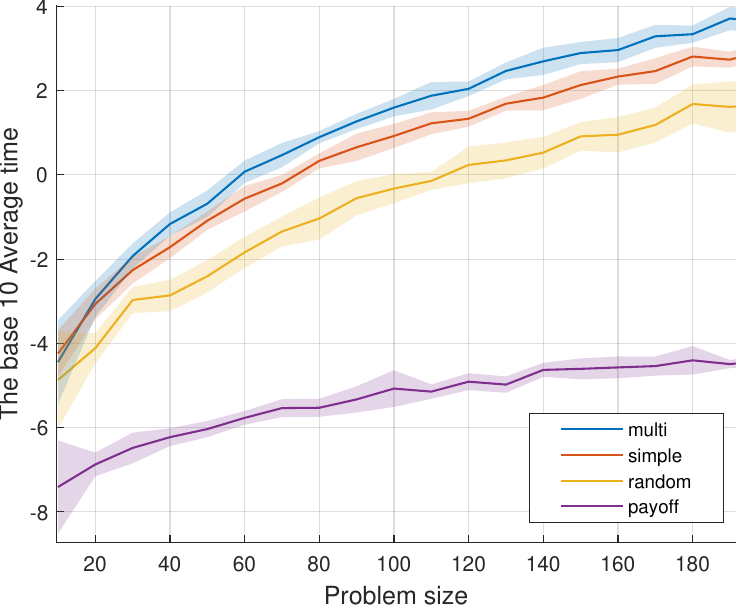}
		\label{fig:p_checkrandom}}
	
	\caption{The average runtime of SEMO, \( \text{EMPMO}_{\text{simple}} \), \( \text{EMPMO}_{\text{random}} \), and \( \text{EMPMO}_{\text{payoff}} \) on artificial problem \( \mathrm{BPAOAZ} \) and the y-axis is the average runtime in base 10.}
	\label{fig:GRcheckR}
\end{figure}




\subsection{On Bi-party Multi-objective UAV Path Planning}



\begin{figure}[htbp]
	\centering
	\small
	\subfloat[UAV1]{\includegraphics[width=0.32\columnwidth]{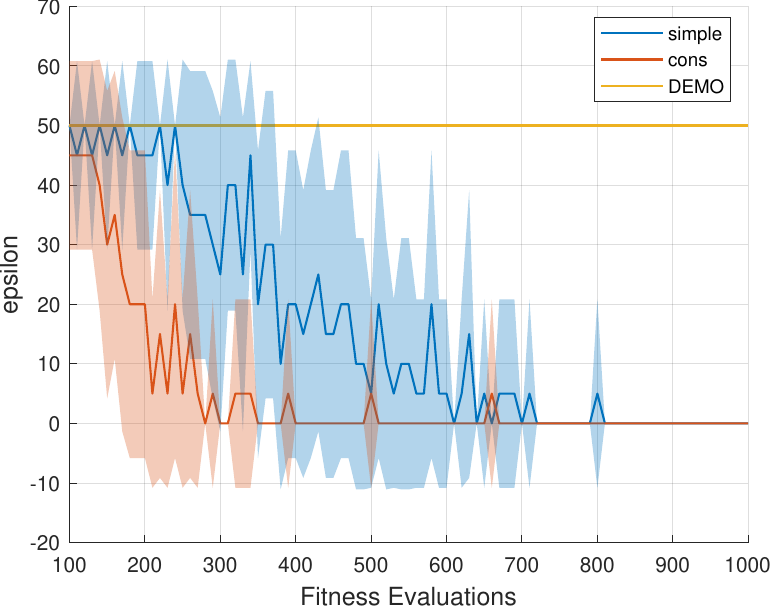}
		\label{fig:UAV1_max}}
	\subfloat[UAV2]{\includegraphics[width=0.32\columnwidth]{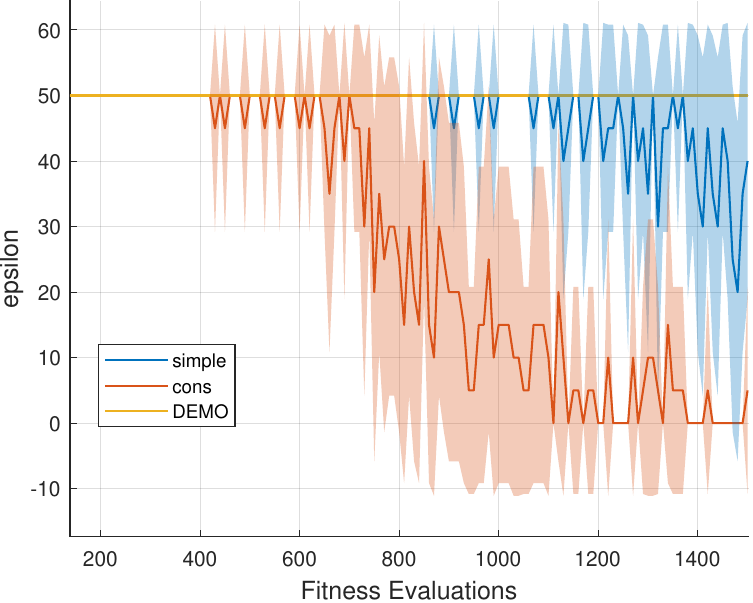}
		\label{fig:UAV3_max}}
	\subfloat[UAV3]{\includegraphics[width=0.32\columnwidth]{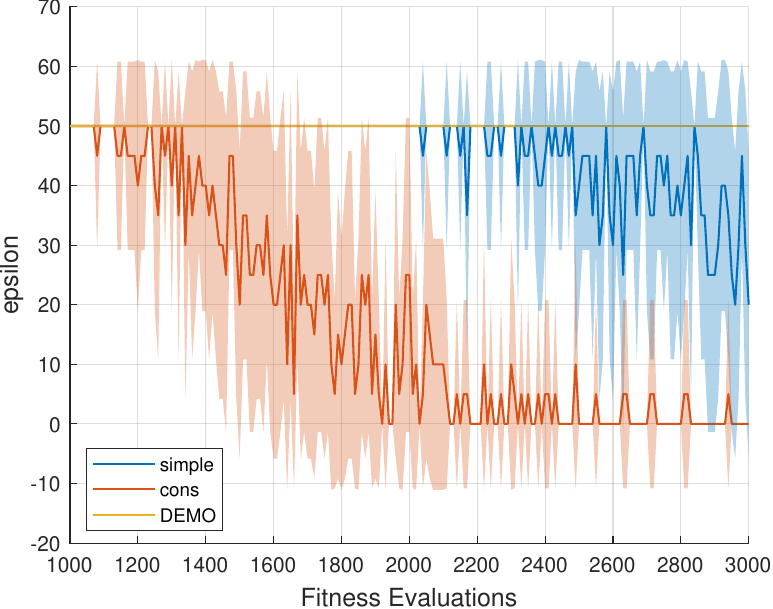}
		\label{fig:UAV5_max}}
	\caption{The largest minimum approximate degree of DEMO, \( \text{EMPMO}^{\text{SP}}_{\text{simple}} \), and \(\text{EMPMO}^{\text{SP}}_{\text{cons}} \) on BPUAVPP.}
	\label{fig:UAV_max}
\end{figure}

In this experiment, we evaluate the performance of DEMO, \( \text{EMPMO}^{\text{SP}}_{\text{simple}} \), and \( \text{EMPMO}^{\text{SP}}_{\text{cons}} \) on the BPMOSP problem, which is based on the bi-party multi-objective UAV path planning (BPUAVPP) problem proposed by Chen et al. \cite{10463192}. The problem involves two stakeholders: the efficiency party and the safety party. To account for the computational complexity of the underlying optimizer, we simplify the BPMOSP problem into a combinatorial optimization problem that seeks the shortest path on a finite set of nodes, disregarding continuous coordinates. Additionally, all constraints in the problem are omitted. The problem construction ensures the existence of a common solution.

\begin{itemize}
	\item The efficiency party focuses on minimizing the path length and the mission hover point distance, formulated as follows:
	\[
	f_{\textup{length}} = \sum_{i=0}^{n-1} ||\mathbf{g}_i||,\quad f_{\textup{distance}} = \sum_{k=0}^{K-1} \min_{i} ||\mathbf{p}_i - \mathbf{p}_k^{\text{job}}||,
	\]
	where \( g_i \) represents the length of the \( i \)-th traversed edge, \( p_k^{\text{job}} \) denotes the \( k \)-th preset UAV hover point, and \( p_i \) is the \( i \)-th discrete trajectory point.
	
	\item The safety party aims to minimize the risks to pedestrians and property, as described by the following objectives:
	
	\[
	f_{\textup{fatal}} = \sum_{i=0}^{n} c_{r_p}(x_i, y_i, z_i), \quad f_{\textup{eco}} = \sum_{i=0}^{n} \psi(z_i),
	\]
	where \( c_{r_p}(x, y, z)= P_{\text{crash}} S_h \sigma_{p}(x, y) R^{P}_f(z) \) represents the fatality risk cost, \( P_{\text{crash}} \) denotes the crash probability, which depends on UAV hardware and software reliability, \( S_h \) represents the crash impact area, \( \sigma_{p}(x, y) \) indicates population density at location \((x, y)\), and \( R^{P}_f(z) \) correlating with the kinetic energy of the impact and obscuration factors, and \( \psi \) is a lognormal distribution function of the flight height \( z_i \).
\end{itemize}



The experiments included three test cases for the BPUAVPP problem, consisting of 10, 30, and 50 vertices, referred to as UAV1, UAV2, and UAV3, respectively. Figures \ref{fig:UAV1_max} - \ref{fig:UAV5_max} depict the worst performance trends of the three algorithms under comparison across these test cases. The x-axis represents the number of evaluations, while the y-axis denotes the largest minimum degree to which all solutions in the population can approximately dominate the common solution for the same endpoint, which can be expressed as:
$
\max_{1\leq i\leq|P|}\varepsilon_{x_i},
$
where $P$ represents the population, and $x_i=(v_0,...,u_j)$ is the $i$-th solution in $P$. Let $\Phi_{u_j}$ represent the set of solutions in the common solution set whose endpoint is $u_j$, then $\varepsilon_{x_i}$ represents the minimum degree to which $x_i$ can approximately dominate all solutions in $\Phi_{u_j}$. In the experiment, the value of $\varepsilon_{x_i}$ is calculated using the bisection method.
This metric reflects the population's worst approximation quality relative to the common solution, with smaller values indicating better population quality.

The results clearly show that for the multi-objective optimization algorithm, the maximum \( \varepsilon \)-value did not converge. This is because the algorithm aims to identify the complete Pareto set (PS) for the multi-objective version of BPUAVPP, including non-common solutions. In contrast, the worst approximations of the other two algorithms converge as the number of evaluations increases, indicating that they are able to find common solutions over time. Among them, our proposed \( \text{EMPMO}^{\text{SP}}_{\text{cons}} \)  is able to find common solutions faster.

In the UAV1 case, a unique common solution can be identified. Since UAV1 consists of only 10 nodes, both the simple and cons algorithms eventually find the true common solution after sufficient evaluations, achieving \( \varepsilon = 0 \). A comparison across different problem complexities reveals that as the problem becomes more challenging, the approximation quality decreases for the same number of evaluations.

These experimental results strongly support the theoretical analysis presented earlier. First, the BPUAVPP problem cannot be effectively solved using multi-objective optimization algorithms, as the population introduces numerous non-common solutions. Second,  \( \text{EMPMO}^{\text{SP}}_{\text{simple}} \), which only enforces agreement between parties in the final stage, requires the algorithm to explore the complete PS during the evolutionary phase to achieve higher precision at the end. This slows the convergence speed. Finally, \( \text{EMPMO}^{\text{SP}}_{\text{cons}} \) not only identifies the correct common solution, as shown in the UAV1 and UAV2 cases, but also achieves higher precision compared to \( \text{EMPMO}^{\text{SP}}_{\text{simple}} \). 

We also conducted experiments on the average minimum degree to which all solutions in the population can approximately dominate the common solution for the same endpoint, and the results strongly supported the previous theoretical analysis. Figures \ref{fig:UAV1} - \ref{fig:UAV5} depict the performance trends of the three algorithms under comparison across these test cases. The x-axis represents the number of evaluations, while the y-axis denotes the average minimum degree to which all solutions in the population can approximately dominate the common solution for the same endpoint. This metric reflects the population's overall approximation quality relative to the common solution, with smaller values indicating better population quality.

\begin{figure}[htbp]
	\centering
	\small
	\subfloat[UAV1]{\includegraphics[width=0.32\columnwidth]{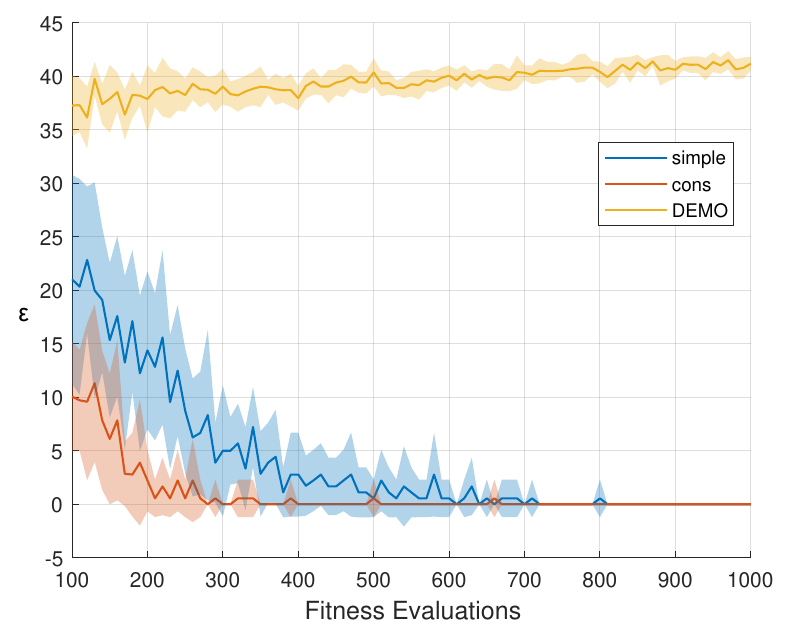}
		\label{fig:UAV1}}
	\subfloat[UAV2]{\includegraphics[width=0.32\columnwidth]{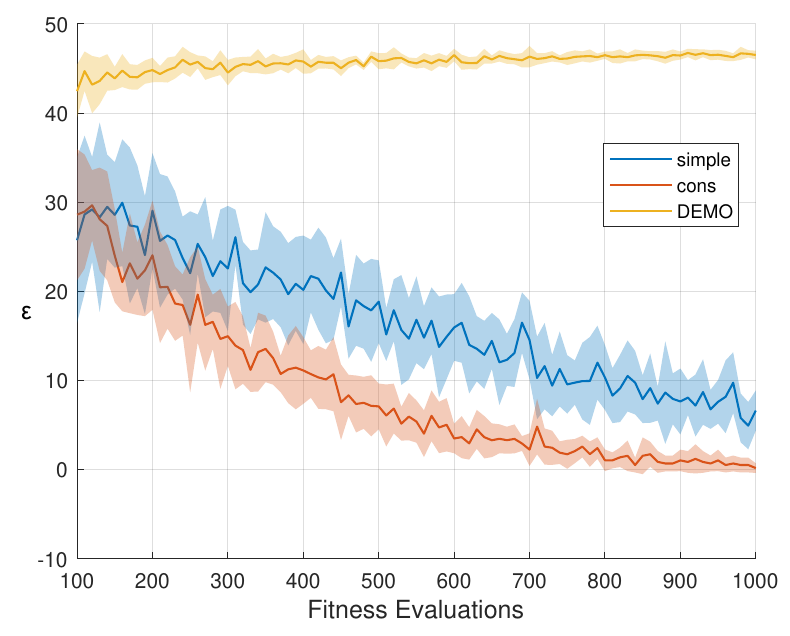}
		\label{fig:UAV3}}
	\subfloat[UAV3]{\includegraphics[width=0.32\columnwidth]{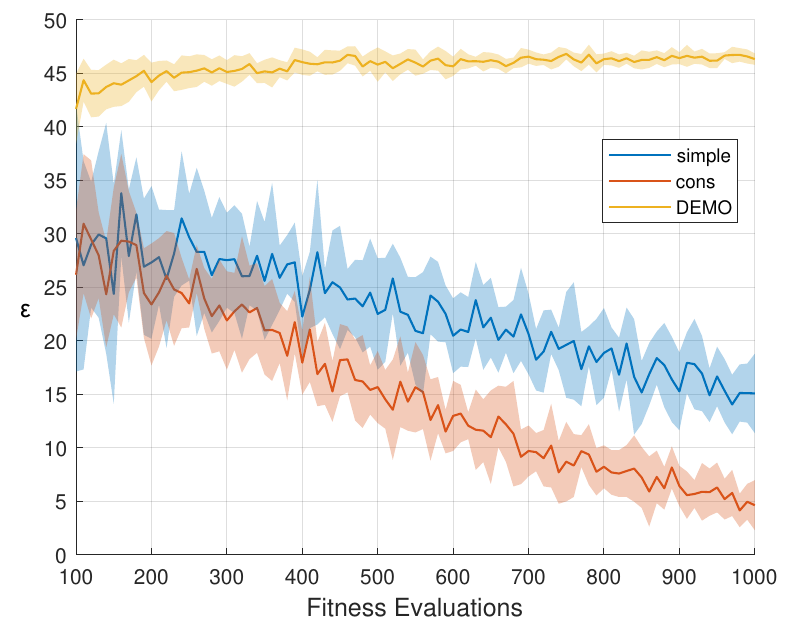}
		\label{fig:UAV5}}
	\caption{The average minimum approximate degree of DEMO, \( \text{EMPMO}^{\text{SP}}_{\text{simple}} \), and \(\text{EMPMO}^{\text{SP}}_{\text{cons}} \) on BPUAVPP.}
	\label{fig:UAV}
\end{figure}

The results in Figures \ref{fig:UAV1} - \ref{fig:UAV5} clearly demonstrate that for multi-objective optimization algorithms, the \( \varepsilon \)-value shows a slight upward trend as the number of evaluations increases. This is because the algorithm aims to identify the complete Pareto set (PS) for the multi-objective version of BPUAVPP, including non-common solutions. As evaluations increase, the inclusion of more non-common solutions leads to a rise in the population's overall \( \varepsilon \)-value. In contrast, the \( \varepsilon \)-values for the other two algorithms decrease over time.

\section{Conclusion}
\label{sec: Conclusion}
In this paper, we present the first mathematical analysis of the runtime of evolutionary algorithms applied to two-party multi-objective optimization problems. We demonstrate that multi-objective optimization algorithms are not suitable for bi-party multi-objective optimization problems, both in terms of runtime and the solution set. We then consider a transition from multi-objective optimization to multi-party multi-objective optimization by decoupling MPMOPs into two MOPs and separately optimizing them to compute their intersection for the common Pareto set. This approach is limited by the NP-hard nature of the problem, making it difficult to obtain exact solutions in polynomial time, and the final solution set composed of approximations often lacks an intersection. Finally, we propose two general bi-party multi-objective optimization frameworks, along with a customized algorithm for the bi-party multi-objective shortest path optimization. We provide theoretical guarantees for an approach that maintains a single population searching for a common solution, and demonstrate that considering the interaction between the two parties in each iteration optimizes both time efficiency and solution set quality. This paper lays the foundation for evolutionary multi-party multi-objective optimization analysis, which will be further extended in the future to analyze bi-party multi-objective optimization problems without a common solution , including the definition of optimality in the case of no common solution, the design of indicators, and specific proof analysis.  In addition, we will also  explore the population size in future work.

\bibliographystyle{IEEEtran}
\bibliography{ref}

\end{document}